\documentclass{article}

    \usepackage[final]{neurips_2023}

\usepackage[utf8]{inputenc} 
\usepackage[T1]{fontenc}    
\usepackage[hidelinks=true]{hyperref}       
\usepackage{url}            
\usepackage{booktabs}       
\usepackage{amsfonts}       
\usepackage{nicefrac}       
\usepackage{microtype}      
\usepackage{xcolor}         
\usepackage{etoc}
\usepackage{comment}

\usepackage{amsmath}
\usepackage{amsthm}
\usepackage{xfrac}
\usepackage{graphicx}
\usepackage{subcaption}
\usepackage{wrapfig}
\usepackage{xcolor}
\usepackage{cleveref}
\usepackage{algorithm, caption}
\usepackage[noend]{algorithmic}
\usepackage{parskip}
\usepackage{xspace}
\usepackage{pifont}
\usepackage{stackrel}
\usepackage{mathtools}
\usepackage[para,online,flushleft]{threeparttable}
\usepackage{adjustbox}
\usepackage{bm}
\usepackage{cleveref}

\newtheorem{lemma}{Lemma}
\newtheorem{corollary}{Corollary}
\newtheorem{definition}{Definition}
\newtheorem{assumption}{Assumption}

\newtheorem{theorem}{Theorem}

\usepackage{xifthen}

\newcommand{\likelihood}[2][]{%
    \ifthenelse{\isempty{#1}}
        {p\left(#2\right)}
        {p_{#1}\left(#2\right)}%
}

\newcommand{\normal}[3][]{%
    \ifthenelse{\isempty{#1}}
        {\mathcal{N}\left(#2, #3\right)}
        {\mathcal{N}\left(#1|#2, #3\right)}%
}

\newcommand{\norm}[1]{\left\lVert#1\right\rVert}
\newcommand{\abs}[1]{\left\lvert#1\right\rvert}

\newcommand{\set}[1]{\left\{#1\right\}}

\DeclareMathOperator*{\argmax}{arg\!\max}

\DeclareMathOperator{\diag}{diag}

\newcommand{\R}{\mathbb{R}}
\newcommand{\Rzero}{\mathbb{R}_{\geq 0}}
\newcommand{\N}{\mathbb{N}}

\newcommand{\E}{\mathbb{E}}
\newcommand{\opacex}{\textsc{OpAx}\xspace}
\newcommand{\ceeus}{CEE-US\xspace}

\newcommand{\capt}[1]{\mdseries{#1}}

\def\mI{{\bm{I}}}

\def\mK{{\bm{K}}}

\def\vzero{{\bm{0}}}

\def\vmu{{\bm{\mu}}}

\def\vtau{{\bm{\tau}}}
\def\veta{{\bm{\eta}}}
\def\va{{\bm{a}}}

\def\vf{{\bm{f}}}

\def\vk{{\bm{k}}}

\def\vu{{\bm{u}}}

\def\vw{{\bm{w}}}
\def\vx{{\bm{x}}}
\def\vy{{\bm{y}}}
\def\vz{{\bm{z}}}
\def\vpi{{\bm{\pi}}}
\def\vmu{{\bm{\mu}}}
\def\vsigma{{\bm{\sigma}}}
\def\hvx{{\hat{\bm{x}}}}

\def\setC{{\mathcal{C}}}
\def\setD{{\mathcal{D}}}

\def\setH{{\mathcal{H}}}

\def\setM{{\mathcal{M}}}
\def\setN{{\mathcal{N}}}
\def\setO{{\mathcal{O}}}

\def\setR{{\mathcal{R}}}

\def\setU{{\mathcal{U}}}

\def\setX{{\mathcal{X}}}

\def\setZ{{\mathcal{Z}}}

\title{Optimistic Active Exploration of Dynamical Systems}

\author{
Bhavya Sukhija$^{1}$ \quad Lenart Treven$^{1}$ \quad Cansu Sancaktar$^2$ \quad Sebastian Blaes$^2$ \\ \textbf{Stelian Coros}$^1$ 
\textbf{Andreas Krause}$^1$ \\
ETH Zürich $^1$ \quad MPI for Intelligent Systems$^2$ \\
\texttt{\{sukhijab,trevenl,scoros,krausea\}@ethz.ch}\\
\texttt{\{cansu.sancaktar,sebastian.blae\}@tuebingen.mpg.de}
}

\begin{document}
\etocdepthtag.toc{mtchapter}
\etocsettagdepth{mtchapter}{subsection}
\etocsettagdepth{mtappendix}{none}

\maketitle

\begin{abstract}
\looseness=-1
Reinforcement learning algorithms commonly seek to optimize policies for solving one particular task. How should we explore an unknown dynamical system such that the estimated model globally approximates the dynamics and allows us to solve multiple downstream tasks in a zero-shot manner? 
In this paper, we address this challenge, by developing an algorithm -- \opacex -- for active exploration. \opacex uses well-calibrated probabilistic models to quantify the epistemic uncertainty about the unknown dynamics. It optimistically---w.r.t.~to plausible dynamics---maximizes the information gain between the unknown dynamics and state observations.  
We show how the resulting optimization problem can be reduced to an optimal control problem that can be solved at each episode using standard approaches. 
We analyze
our algorithm for general models, and, in the case of Gaussian process dynamics, we give a first-of-its-kind sample complexity bound and
show that the epistemic uncertainty \emph{converges to zero}. 
In our experiments, we compare \opacex with other heuristic active exploration approaches on several environments.
Our experiments show that \opacex is not only theoretically sound but also performs well for zero-shot planning on novel downstream tasks.
\end{abstract}

\section{Introduction}
\looseness -1 Most reinforcement learning (RL) algorithms are designed to maximize cumulative rewards for a single task at hand. Particularly, model-based RL algorithms, such as~\citep{chua2018pets, kakade2020information, curi2020efficient}, excel in efficiently exploring the dynamical system as they direct the exploration in regions with high rewards. However, due to the directional bias, their underlying learned dynamics model fails to generalize in other areas of the state-action space. While this is sufficient if only one control task is considered, it does not scale to the setting where the system is used to perform several tasks, i.e., under the same dynamics optimized for different reward functions.
As a result, when presented with a new reward function, they often need to relearn a policy from scratch, requiring many interactions with the system, or employ multi-task~\citep{zhang2021survey} or transfer learning~\citep{weiss2016survey} methods. Traditional control approaches such as trajectory optimization~\citep{trajectorPlanningForRobots} and model-predictive control~\citep{mpc} assume knowledge of the system's dynamics. 
They leverage the dynamics model to solve an optimal control problem for each task. 
Moreover, in the presence of an accurate model, important system properties such as stability and sensitivity can also be studied. Hence, knowing an accurate dynamics model bears many practical benefits.
However, in many real-world settings, obtaining a model using just physics' first principles is very challenging. A promising approach is to leverage data for learning the dynamics, i.e., system identification or active learning.
To this end, the key question we investigate in this work is: \emph{how should we interact with the system to learn its dynamics efficiently?} 
    
\looseness -1 
While active learning for regression and classification tasks is well-studied, active learning in RL is much less understood.
    In particular, active learning methods that yield strong theoretical and practical results, generally query data points based on information-theoretic criteria~\citep{krause2008near,settles2009active, balcan2010true,hanneke2014theory, chen15mis}. In the context of dynamical systems, this requires querying arbitrary transitions~\citep{berkenkamp2017safe, mehta2021experimental}.
    However, in most cases, \emph{querying a dynamical system at any state-action pair is unrealistic}. Rather, we can only execute policies on the real system and observe the resulting trajectories. Accordingly, an active learning algorithm for RL needs to suggest policies that are ``informative'' for learning the dynamics. This is challenging since it requires planning with unknown dynamics. 
    
    \paragraph{Contributions}
    In this paper, we introduce a new algorithm, {\em Optimistic Active eXploration (\opacex)}, designed to actively learn nonlinear dynamics within continuous state-action spaces. During each episode, \opacex plans an exploration policy to gather the most information possible about the system. It learns a statistical dynamics model that can quantify its epistemic uncertainty and utilizes this uncertainty for planning. The planned trajectory targets state-action pairs where the model's epistemic uncertainty is high, which naturally encourages exploration. 
    In light of unknown dynamics, \opacex uses an optimistic planner that picks policies that optimistically yield maximal information. We show that this \emph{optimism paradigm} plays a crucial role in studying the theoretical properties of \opacex. Moreover,
 we provide a general convergence analysis for \opacex and prove convergence to the true dynamics for Gaussian process (GP) dynamics models.  Theoretical guarantees for active learning in RL exist for a limited class of systems~\citep{simchowitz2018learning, wagenmaker2020active, mania2020active}, but lack for a more general and practical class of dynamics~\citep{chakraborty2023steering, wagenmaker2023optimal}. 
 We are, to the best of our knowledge, the first to give convergence guarantees for a rich class of nonlinear dynamical systems. 

 We evaluate \opacex on several simulated robotic tasks with state dimensions ranging from two to 58. The empirical results provide validation for our theoretical conclusions, showing that \opacex consistently delivers strong performance across all tested environments.
 Finally, we provide an efficient implementation\footnote{\url{https://github.com/lasgroup/opax}} of \opacex in JAX \citep{jax2018github}.   

\section{Problem Setting}
\label{problem_setting}
We study an \emph{unknown} discrete-time dynamical system $\vf^*$, with state $\vx \in \setX \subset \R^{d_x}$ and control inputs $\vu \in \setU \subset \R^{d_u}$.
\begin{equation}
    \vx_{k+1} = \vf^*(\vx_k, \vu_k) + \vw_k. 
    \label{eq:system_eq}
\end{equation}
Here, $\vw_k$ represents the stochasticity of the system for which we assume  $\vw_k \stackrel{\mathclap{i.i.d.}}{\sim} \normal{\vzero}{\sigma^2\mI}$ (\Cref{ass:noise_properties}). Most common approaches in control, such as trajectory optimization and model-predictive control (MPC), assume that the dynamics model $\vf^*$  is known and leverages the model to control the system state. Given a cost function $c: \setX \times \setU \to \R$, such approaches formulate and solve an optimal control problem to obtain a sequence of control inputs that drive the system's state
\begin{align}
    \underset{\vu_{0:T-1}}{\arg\min} \hspace{0.2em} &\E_{\vw_{0:T-1}}\left[\sum^{T-1}_{t=0} c(\vx_t, \vu_t) \right],         \label{eq:OC_problem} \\
        \vx_{t+1} &= \vf^*(\vx_t, \vu_t) + \vw_t \quad \forall 0\leq t\leq T. \notag
\end{align}
Moreover, if the dynamics are known, many important characteristics of the system such as stability, and robustness ~\citep{khalil2015nonlinear} can be studied.
However, in many real-world scenarios, an accurate dynamics model $\vf^*$ is not available. Accordingly, in this work, we consider the problem of actively learning the dynamics model from data a.k.a.~system identification~\citep{ASTROMSystemID}. Specifically, we are interested in devising a cost-agnostic algorithm that focuses solely on learning the dynamics model. Once a good model is learned, it can be used for solving different downstream tasks by varying the cost function in~\Cref{eq:OC_problem}. 

We study an episodic setting, with episodes $n=1, \ldots, N$. At the beginning of the episode $n$, we deploy an exploratory policy $\vpi_n$, chosen from a policy space $\Pi$ for a horizon of $T$ on the system.  Next, we obtain trajectory $\vtau_n = (\vx_{n,0}, \ldots, \vx_{n, T})$, which we save to a dataset of transitions  $\setD_n = \set{(\vz_{n, i } = (\vx_{n, i}, \vpi_n(\vx_{n, i})), \vy_{n, i} = \vx_{n, i + 1})_{0 \le i < T}}$.
We use the collected data to learn an estimate $\vmu_n$ of $\vf^*$. 
To this end, the goal of this work is to propose an algorithm $\textit{Alg}$, that at each episode $n$ leverages the data acquired thus far, i.e., $\setD_{1:n-1}$ to determine a policy $\vpi_n \in \Pi$ for the next step of data collection, that is,  $\textit{Alg}(\setD_{1:n-1}, n) \to \vpi_n$. The proposed algorithm should be consistent, i.e., $\vmu_n(\vz) \to \vf^*(\vz)$ for $n \to \infty$ for all $\vz \in \setR$, where $\setR$ is the reachability set defined as
\begin{equation*}
   \setR = \{ \vz \in \setZ \mid \exists (\vpi \in \Pi, t \leq T), \text{ s.t.}, p(\vz_t=\vz|\vpi, \vf^*) > 0\}, 
\end{equation*} 
and efficient w.r.t.~rate of convergence of $\vmu_n$ 
 to $\vf^*$.

 To devise such an algorithm, we take inspiration from Bayesian experiment design ~\citep{chaloner1995bayesian}. In the Bayesian setting, given a prior over $\vf^*$, a natural objective for active exploration is the mutual information~\citep{lidley_info_gain} between $\vf^*$ and observations $\vy_{\setD_n}$.
\begin{definition}[Mutual Information, \cite{elementsofIT}]
\label{def:MI}
The mutual information between $\vf^{*}$ and its noisy measurements $\vy_{\setD_n}$ for points in $\setD_n$, where $\vy_{\setD_n}$ is the concatenation of $(\vy_{\setD_n, i})_{i<T}$ is defined as,
\begin{align}
  F(\setD_n) := I \left({\vf^{*}}; {\vy_{\setD_n}}\right) = H\left(\vy_{\setD_n}\right) - H\left( \vy_{\setD_n}\mid\vf^{*}\right),
\end{align}
where $H$ is the Shannon differential entropy. 
\end{definition}
The mutual information quantifies the reduction in entropy of $\vf^*$ conditioned on the observations.
Hence, maximizing the mutual information w.r.t.~the dataset $\setD_{n}$ leads to the maximal entropy reduction of our prior. Accordingly, a natural objective for active exploration in RL can be the mutual information between $\vf^*$ and the collected transitions over a budget of $N$ episodes, i.e., $I \left({\vf^{*}}; {\vy_{\setD_{1:N}}}\right)$. This requires maximizing the mutual information over a sequence of policies, which is a challenging planning problem even in settings where the dynamics are known~\citep{pmlr-v206-mutny23a}. A common approach is to greedily pick a policy that maximizes the information gain conditioned on the previous observations at each episode:
\begin{equation}
    \underset{\vpi \in \Pi}{\max} \hspace{0.2em} \E_{\tau^{\vpi}}\left[I\left(\vf^*_{\tau^{\vpi}}; \vy_{\tau^{\vpi}} \mid \setD_{1:n-1}\right)\right].
    \label{eq:greedy_info_gain}
\end{equation}
Here $\vf^*_{\tau^{\vpi}} = (\vf^*(\vz_{n, 0}), \ldots, \vf^*(\vz_{n, T-1}))$, $\vy_{\tau^{\vpi}} = (\vy_{n, 0},\ldots, \vy_{n, T-1})$, $\tau^{\vpi}$ is the trajectory under the policy $\vpi$, and the expectation is taken w.r.t.~the process noise $\vw$.
\paragraph{Interpretation in frequentist setting}
\looseness=-1
While information gain is Bayesian in nature (requires a prior over $\vf^*$), it also has a frequentist interpretation. In particular, later in \Cref{sec:opacex} we relate it to the epistemic uncertainty of the learned model. Accordingly, while this notion of information gain stems from Bayesian literature, we can use it to motivate our objective in both Bayesian and frequentist settings.
\subsection{Assumptions}
In this work, we learn a probabilistic model of the function $\vf^{*}$ from data. Moreover, at each episode $n$, we learn the mean estimator $\vmu_n(\vx, \vu)$ and the epistemic uncertainty $\vsigma_n(\vx, \vu)$, which quantifies our uncertainty on the mean prediction. 
To this end, we use Bayesian models such as Gaussian processes \citep[GPs,][]{rasmussen2005gp} or Bayesian neural networks  \citep[BNNs,][]{bnnsurvey}. More generally, we assume our model is {\em well-calibrated}:
\begin{definition}[All-time calibrated statistical model of $\vf^*$, \citet{rothfuss2023hallucinated}]
Let, $\vz = (\vx, \vu)$ and $\setZ := \setX \times \setU$. 
 An all-time calibrated statistical model of the function $\vf^*$ is a sequence $\left(\vmu_n, \vsigma_n, \beta_n(\delta)\right)_{n\geq 0}$, such that
\begin{align*}
    \Pr \left(\forall \vz \in \setZ, \forall l \in \set{1, \ldots, d_x}, \forall n \in \N: \abs{\mu_{n, l}(\vz) - f_l(\vz)} \le \beta_n(\delta) \sigma_{n, l}(\vz) \right) \geq 1 - \delta 
\end{align*}
 Here $\mu_{n, l}$ and $\sigma_{n, l}$ are the l-th element in the vector valued functions $\vmu_n$ and $\vsigma_n$ respectively. The scalar function, $\beta_n(\delta) \in \Rzero$ quantifies the width of the $1-\delta$ confidence intervals. We assume w.l.o.g.~that $\beta_n$ monotonically increases with $n$, and that $\sigma_{n, l}(\vz) \leq \sigma_{\max}$ for all $\vz \in \setZ$, $n \geq 0$, and $l \in \set{1, \ldots, d_x}$.
 \label{def:all_time_calibrated_model}
\end{definition}
\begin{assumption}[Well calibration assumption]
\label{assumption: Well Calibration Assumption}
   Our learned model is an all-time-calibrated statistical model of $\vf^*$, i.e., there exists a sequence of $\left(\beta_{n}(\delta)\right)_{n\geq0}$ such that our model satisfies the well-calibration condition, c.f., \Cref{def:all_time_calibrated_model}.
\end{assumption}
This is a natural assumption on our modeling. It states that we can make a mean prediction and also quantify how far it is off from the true one with high probability. A GP model satisfies this requirement for a very rich class of functions, c.f., \Cref{lem:rkhs_confidence_interval}. For BNNs, calibration methods \citep{kuleshov2018accurate} are often used and perform very well in practice. 
Next, we make a simple continuity assumption on our function $\vf^*$.
\begin{assumption}[Lipschitz Continuity]
\label{ass:lipschitz_continuity}
The dynamics model $\vf^*$ and our epistemic uncertainty prediction $\vsigma_n$ are $L_{\vf}$ and $L_{\vsigma}$ Lipschitz continuous, respectively. Moreover,
we define $\Pi$ to be the policy class of $L_{\vpi}$ Lipschitz continuous functions. 
\end{assumption}
The Lipschitz continuity assumption on $\vf^*$ is quite common in control theory~\citep{khalil2015nonlinear} and learning literature~\citep{curi2020efficient,pasztor2021efficient,sussex2022model}. Furthermore, the Lipschitz continuity of $\vsigma_n$ also holds for GPs with common kernels such as the linear or radial basis function (RBF) kernel~\citep{rothfuss2023hallucinated}.

Finally, we reiterate the assumption of the system's stochasticity.
\begin{assumption}[Process noise distribution]
\looseness=-1
The process noise is i.i.d. Gaussian with variance $\sigma^2$, i.e., $\vw_k \stackrel{\mathclap{i.i.d}}{\sim} \setN(\vzero, \sigma^2\mI)$.
\label{ass:noise_properties}
\end{assumption}
We focus on the setting where $\vw$ is homoscedastic for simplicity. However, our framework can also be applied to the more general heteroscedastic and sub-Gaussian case (c.f., \Cref{thm:gp_theorem}).

\section{Optimistic Active Exploration}\label{sec:opacex}
In this section, we propose our \emph{optimistic active exploration} (\opacex) algorithm. The algorithm consists of two main contributions: \emph{(i)} First we reformulate the objective in \Cref{eq:greedy_info_gain} to a simple optimal control problem, which suggests policies that visit states with high epistemic uncertainty.
\emph{(ii)} We leverage the optimistic planner introduced by \citet{curi2020efficient} to efficiently plan a policy under  \emph{unknown} dynamics. Moreover, we show that the optimistic planner is crucial in giving theoretical guarantees for the algorithm.
\subsection{Optimal Exploration Objective}
The objective in \Cref{eq:greedy_info_gain} is still difficult and expensive to solve in general. However, since in this work, we consider Gaussian noise, c.f., \Cref{ass:noise_properties}, we can simplify this further.
\begin{lemma}[Information gain is upper bounded by sum of epistemic uncertainties]\label{lem:uniform_exploration_upper}
  Let $\vy = \vf^*(\vz) + \vw$, with $\vw \sim \setN(0, \sigma^2\mI)$ and let $\vsigma_{n-1}$ be the epistemic uncertainty after episode $n-1$. Then the following holds for all $n\geq 1$ and dataset $\setD_{1:n-1}$, 
  \begin{align}
      \label{eq:tractable_obj}
    I\left(\vf^*_{\vtau^{\vpi}}; \vy_{\vtau^{\vpi}} \mid \setD_{1:n-1}\right) \leq \frac{1}{2} \sum_{t=0}^{T-1} \sum^{d_x}_{j=1}\log \left(1 + \frac{\sigma_{n-1, j}^2(\vz_t)}{\sigma^2}\right).
  \end{align}
  \end{lemma}
  \looseness=-1
  We prove \Cref{lem:uniform_exploration_upper} in \Cref{sec:proofs}.  
  The information gain is non-negative \citep{elementsofIT}. Therefore, if the right-hand side of \Cref{eq:tractable_obj} goes to zero, the left-hand side goes to zero as well. \Cref{lem:uniform_exploration_upper}  relates the information gain to the model epistemic uncertainty. Therefore, it gives a tractable objective that also has a frequentist interpretation - collect points with the highest epistemic uncertainty. We can use it to plan a trajectory at each episode $n$, by solving the following optimal control problem:
  \begin{align}
      \vpi^*_n = \argmax_{\vpi \in \Pi}& \hspace{0.2em} J_n(\vpi) = \argmax_{\vpi \in \Pi} \hspace{0.2em} \E_{\vtau^{\vpi}}\left[\sum_{t=0}^{T-1} \sum^{d_x}_{j=1} \log \left(1 + \frac{\sigma_{n-1, j}^2(\vx_t, \vpi(\vx_t))}{\sigma^2}\right)\right],       \label{eq:exploration_op} \\
      \vx_{t+1} &= \vf^{*}(\vx_t, \vpi(\vx_t)) + \vw_t. \notag
  \end{align}
  The problem in \Cref{eq:exploration_op} is closely related to previous literature in active exploration for RL. For instance, some works consider different geometries such as the sum of epistemic uncertainties (\cite{pathak2019self, sekar2020planning}, c.f., \cref{sec:study_intrinsic_reward} for more detail).
\subsection{Optimistic Planner}
The optimal control problem in \Cref{eq:exploration_op} requires knowledge of the dynamics $\vf^*$ for planning, however, $\vf^*$ is unknown.
A common choice is to use the mean estimator $\vmu_{n-1}$ in \Cref{eq:exploration_op} instead of $\vf^*$ for planning~\citep {active_learning_gp}. However, in general, using the mean estimator is susceptible to model biases~\citep{chua2018pets} and is provably optimal only in the case of linear systems~\citep{linearExploration}. 
To this end, we propose using an optimistic planner, as suggested in \cite{curi2020efficient}, instead. Accordingly, given the mean estimator $\vmu_{n-1}$ and the epistemic uncertainty $\vsigma_{n-1}$, we solve the following optimal control problem
\begin{align}
   \vpi_n,  \veta_n &= \argmax_{\vpi \in \Pi, \veta \in \Xi} J_n(\vpi, \veta) = \argmax_{\vpi \in \Pi, \veta \in \Xi} \E_{\vtau^{\vpi, \veta}}\left[\sum_{t=0}^{T-1} \sum^{d_x}_{j=1} \log \left(1 + \frac{\sigma_{n-1, j}^2(\hvx_t, \vpi(\hvx_t))}{\sigma^2}\right)\right],       \label{eq:exploration_op_optimistic} \\
      \hvx_{t+1} &= \vmu_{n-1}(\hvx_t, \vpi(\hvx_t)) + \beta_{n-1}(\delta)\vsigma_{n-1}(\hvx_t, \vpi(\hvx_t)) \veta(\hvx_t) + \vw_t, \notag
\end{align}
where $\Xi$ is the space of policies $\veta: \setX \to [-1, 1]^{d_x}$.
Therefore, we use the policy $\veta$ to ``hallucinate'' (pick) transitions that give us the most information. Overall, the resulting formulation corresponds to a simple optimal control problem with a larger action space, i.e., we increase the action space by another $d_x$ dimension. A natural consequence of \Cref{assumption: Well Calibration Assumption} is that $J_n(\vpi_n^*) \leq J_n(\vpi_n,  \veta_n)$ with high probability (c.f.,~\Cref{cor:optimistic_estimate} in \Cref{sec:proofs}). That is by solving \Cref{eq:exploration_op_optimistic}, we get an optimistic estimate on \Cref{eq:exploration_op}.
Intuitively, the policy $\vpi_{n}$ that \opacex suggests, behaves optimistically with respect to the information gain at each episode. 
\begin{algorithm}[t]
    \caption*{\textbf{\opacex: } \textsc{Optimistic Active Exploration}}
    \begin{algorithmic}[]
        \STATE {\textbf{Init:}}{ Aleatoric uncertainty $\sigma$, Probability $\delta$, Statistical model $(\vmu_0, \vsigma_0, \beta_0(\delta))$}
        \FOR{episode $n=1, \ldots, N$}{
            \vspace{-0.5cm}
            \STATE {
            \begin{align*}
                &\vpi_n = \argmax_{\vpi \in \Pi}\max_{\veta \in \Xi} \E\left[\sum_{t=0}^{T-1} \sum^{d_x}_{j=1} \log \left(1 + \frac{\sigma_{n-1, j}^2(\vx_t, \vpi(\vx_t))}{\sigma^2}\right)\right] &&\text{\ding{228} Prepare policy} \\
                &\mathcal{D}_n \leftarrow \textsc{Rollout}(\vpi_n) \quad &&\text{\ding{228} Collect measurements } \\
                 &\text{Update } (\vmu_n, \vsigma_n, \beta_n(\delta)) \leftarrow \setD_{1:n} \quad &&\text{\ding{228} Update model}
            \end{align*}
            }}
        \ENDFOR
    \end{algorithmic}
\end{algorithm}

\section{Theoretical Results} \label{sec:theoretical_results}
We theoretically analyze the convergence properties of \opacex. We first study the regret of planning under unknown dynamics. Specifically, since we cannot evaluate the optimal exploration policy from \cref{eq:exploration_op} and use the optimistic one, i.e., \cref{eq:exploration_op_optimistic} instead, we incur a regret. We show that due to the optimism in the face of uncertainty paradigm, we can give sample complexity bounds for the Bayesian and frequentist settings. All the proofs are presented in \Cref{sec:proofs}.
\begin{lemma}[Regret of optimistic planning under unknown dynamics]\label{lem:optimistic_planning_regret}
Let \Cref{assumption: Well Calibration Assumption} hold. 
Furthermore, define  $J_{n, k}(\vpi_n, \veta_n, \vx)$ as 
\begin{align*}
    J_{n, k}(\vpi_n, \veta_n, \vx) &= \E_{\vtau^{\vpi_n, \veta_n}}\left[\sum_{t=k}^{T-1} \sum^{d_x}_{j=1} \log \left(1 + \frac{\sigma_{n-1, j}^2(\hvx_t, \vpi_n(\hvx_t))}{\sigma^2}\right)\right],   \\ &\text{ s.t. }  \hvx_{t+1} = \vmu_{n-1}(\hvx_t, \vpi_n(\hvx_t)) + \beta_{n-1}(\delta)\vsigma_{n-1}(\hvx_t, \vpi_n(\hvx_t)) \veta_n(\hvx_t) + \vw_t \\
      &\text{ and } \hvx_0 = \vx.
\end{align*}
Then, for all $n \geq 1$, with probability at least $1-\delta$,
\begin{align*}
    J_n(\vpi_n^*) - J_n(\vpi_n) &\leq \sum^{T-1}_{t=0} \E_{\vtau^{\vpi_n}}\left[J_{n, t+1}(\vpi_n, \veta_n, \vx'_{t+1}) - J_{n, t+1}(\vpi_n, \veta_n, \vx_{t+1}) \right], \\
    &\text{with } \vx_{t+1} = \vf^*(\vx_{t}, \vpi_n(\vx_t)) + \vw_t, \\
    &\text{and } \vx'_{t+1} = \vmu_{n-1}(\vx_t, \vpi_n(\vx_t)) + \beta_{n-1}(\delta)\vsigma_{n-1}(\vx_t, \vpi_n(\vx_t)) \veta_n(\vx_t) + \vw_t.
\end{align*}
\end{lemma}
\Cref{lem:optimistic_planning_regret} gives a bound on the regret of planning optimistically under unknown dynamics. The regret is proportional to the difference in the expected returns for $\vx_t$ and $\vx'_t$. Note, $\norm{\vx_t - \vx'_t} \propto \beta_n(\delta)\vsigma_{n-1}(\vx_{t-1}, \vpi_n(\vx_{t-1}))$. Hence, when we have low uncertainty in our predictions, planning optimistically suffers smaller regret. Next, we leverage \Cref{lem:optimistic_planning_regret} to give a sample complexity bound for the Bayesian and frequentist setting.
\paragraph{Bayesian Setting}
We start by introducing a measure of model complexity as defined by \citet{curi2020efficient}.
\begin{equation}
    {\setM\setC}_N(\vf^*) := \underset{\setD_1, \dots, \setD_N \subset \setZ \times \setX}{\max}\sum^{N}_{n=1} \sum_{\vz \in \setD_n} \norm{\vsigma_{n-1}(\vz)}^2_2.
\end{equation}
This complexity measure captures the difficulty of learning $\vf^*$ given $N$ trajectories. Mainly, the more complicated $\vf^*$, the larger the epistemic uncertainties $\vsigma_n$, and in turn, the larger corresponding ${\setM\setC}_N(\vf^*)$. Moreover, if the model complexity measure is sublinear in $N$, i.e. ${\setM\setC}_N(\vf^*)/N \rightarrow 0$ for $N\rightarrow \infty$, then the epistemic uncertainties also converge to zero in the limit, 
which implies convergence to the true function $\vf^*$.
We present our main theoretical result, in terms of the model complexity measure.
\begin{theorem}\label{thm:main_theorem}
Let \Cref{assumption: Well Calibration Assumption} and \ref{ass:noise_properties} hold. Then, for all $N \geq 1$, with probability at least $1-\delta$,
\begin{equation}
    \E_{\setD_{1:N-1}}\left[\max_{\vpi \in \Pi}\E_{\vtau^{\vpi}}\left[I\left(\vf^*_{\vtau^{\vpi}}; \vy_{\vtau^{\vpi}} \mid \setD_{1:N-1}\right)\right]\right] \le \setO \left(\beta_N T^{\sfrac{3}{2}}\sqrt{\frac{{\setM\setC}_N(\vf^*)}{N}}\right)
    \label{eq:bound_general}
\end{equation}
\end{theorem}
\looseness=-1
\Cref{thm:main_theorem} relates the maximum expected information gain at iteration $N$ to the model complexity of our problem. For deterministic systems, the expectation w.r.t.~$\vtau^{\vpi}$ is redundant.
The bound in \Cref{eq:bound_general} depends on the Lipschitz constants, planning horizon, and dimensionality of the state space (captured in $\beta_{N}$ and ${\setM\setC}_N(\vf^*)$). 
If the right-hand side is monotonically decreasing with $N$, \Cref{thm:main_theorem} guarantees that the information gain at episode $N$ is also shrinking with $N$, and the algorithm is converging. 
Empirically, \cite{pathak2019self} show that the epistemic uncertainties go to zero as more data is acquired.
In general, deriving a worst-case bound on the model complexity is a challenging and active open research problem. However, in the case of GPs, convergence results can be shown for a very rich class of functions. We show this in the following for the frequentist setting.
\looseness=-1
\paragraph{Frequentist Setting with Gaussian Process Models}
We extend our analysis to the frequentist kernelized setting, where $\vf^*$ resides in a Reproducing Kernel Hilbert Space (RKHS) of vector-valued functions.
\begin{assumption}
We assume that the functions $f^*_j$, $j \in \set{1, \ldots, d_x}$ lie in a RKHS with kernel $k$ and have a bounded norm $B$, that is $\vf^* \in \setH^{d_x}_{k, B}$, with $\setH^{d_x}_{k, B} = \{\vf \mid \norm{f_j}_k \leq B, j=1, \dots, d_x\}$.
\label{ass:rkhs_func}
\end{assumption}
In this setting, we model the posterior mean and epistemic uncertainty of the vector-valued function $\vf^*$ with ${\bm \mu}_n(\vz) = [\mu_{n,j} (\vz)]_{j\leq d_x}$, and $\vsigma_n(\vz) = [\sigma_{n,j} (\vz)]_{j\leq d_x}$, where,
\begin{equation}
\begin{aligned}
\label{eq:GPposteriors}
        \mu_{n,j} (\vz)& = {\bm{k}}_{n}^\top(\vz)({\bm K}_{n} + \sigma^2 \bm{I})^{-1}\vy_{1:n}^j,  \\
     \sigma^2_{n, j}(\vz) & =  k(\vx, \vx) - {\bm k}^\top_{n}(\vz)({\bm K}_{n}+\sigma^2 \bm{I})^{-1}{\bm k}_{n}(\vx),
\end{aligned}
\end{equation}
Here, $\vy_{1:n}^j$ corresponds to the noisy measurements of $f^*_j$, i.e., the observed next state from the transitions dataset $\setD_{1:n}$, $\vk_n = [k(\vz, \vz_i)]_{i\leq nT}, \vz_i \in \setD_{1:n}$, and $\bm{K}_n = [k(\vz_i, \vz_l)]_{i, l\leq nT}, \vz_i, \vz_l \in \setD_{1:n}$ is the data kernel matrix. 
It is known that if $\vf^*$ satisfies \Cref{ass:rkhs_func}, then \Cref{eq:GPposteriors} yields well-calibrated confidence intervals, i.e.,  that \Cref{assumption: Well Calibration Assumption} is satisfied. 
\begin{lemma}[Well calibrated confidence intervals for RKHS, \citet{rothfuss2023hallucinated}]
    Let $\vf^* \in \setH_{k,B}^{d_x}$.
Suppose ${\vmu}_n$ and $\vsigma_n$ are the posterior mean and variance of a GP with kernel $k$, c.f., \Cref{eq:GPposteriors}.
There exists $\beta_n(\delta)$, for which the tuple $(\vmu_n, \vsigma_n, \beta_n(\delta))$ satisfies~\Cref{assumption: Well Calibration Assumption} w.r.t. function $\vf^*$.
\label{lem:rkhs_confidence_interval}
\end{lemma}
\Cref{thm:gp_theorem} presents our convergence guarantee for the kernelized case to the $T$-step reachability set $\setR$ for the policy class $\pi \in \Pi$. In particular, $\setR$ is defined as
\begin{equation*}
   \setR = \{ \vz \in \setZ \mid \exists (\vpi \in \Pi, t \leq T), \text{ s.t.}, p(\vz_t=\vz|\vpi, \vf^*) > 0\} 
\end{equation*}
There are two key differences from \Cref{thm:main_theorem}; (\emph{i}) we can derive an upper bound on the epistemic uncertainties $\vsigma_n$, and (\emph{ii}) we can bound the model complexity ${\setM \setC}_{N}(\vf^*)$, with the {\em maximum information gain} of kernel $k$ introduced by \citet{srinivas}, defined as
\begin{equation*}
    {\gamma}_{N}(k) = \max_{\setD_1, \dots, \setD_N; |\setD_n| \leq T}  \frac{1}{2}\log\det(\mI + \sigma^{-2} {\bm K}_N).
\end{equation*}
\begin{theorem}
\label{thm:gp_theorem}
Let \Cref{ass:noise_properties} and \ref{ass:rkhs_func} hold, Then, for all $N \geq 1$, with probability at least $1-\delta$,
\begin{equation}
   \max_{\vpi \in \Pi} \E_{\vtau^{\vpi}}\left[ \max_{\vz \in \vtau^{\vpi}}\sum^{d_x}_{j=1} \frac{1}{2}\sigma_{N, j}^2(\vz) \right] \le \setO \left(\beta_N T^{\sfrac{3}{2}}\sqrt{\frac{{\gamma}_{N}(k)}{N}}\right)
    \label{eq:bound_gp_general}.
\end{equation}
If we relax noise \Cref{ass:noise_properties} to $\sigma$-sub Gaussian. 
Then, if \Cref{ass:lipschitz_continuity} holds, we have for all $N \geq 1$, with probability at least $1-\delta$,
\begin{equation}
   \max_{\vpi \in \Pi} \E_{\vtau^{\vpi}}\left[ \max_{\vz \in \vtau^{\vpi}}\sum^{d_x}_{j=1} \frac{1}{2}\sigma_{N, j}^2(\vz) \right] \le \setO \left(\beta^{T}_{N} T^{\sfrac{3}{2}}\sqrt{\frac{{\gamma}_{N}(k)}{N}}\right) \label{eq:bound_gp}.
\end{equation}
Moreover, 
if $\gamma_N(k) = \setO\left(\text{polylog}(N)\right)$, then for all $\vz \in \setR$, and $1 \leq j \leq d_x$,
    \begin{equation}
    \sigma_{N, j}(\vz) \xrightarrow[]{\mathrm{a.s.}} 0 \text{ for } N \to \infty \label{eq:reachabilty_set_convergence}.
\end{equation}
\end{theorem}
We only state \Cref{thm:gp_theorem} for the expected epistemic uncertainty along the trajectory at iteration $N$. For deterministic systems, the expectation is redundant and for stochastic systems, we can leverage concentration inequalities to give a bound without the expectation (see \Cref{sec:proofs} for more detail).

For the Gaussian noise case, we obtain a tighter bound by leveraging the change of measure inequality from \citet[Lemma~C.2.]{kakade2020information} (c.f.,~\Cref{lem:absolute_diff_gaussians} in \Cref{sec:proofs} for more detail). In the more general case of sub-Gaussian noise, we cannot use the same analysis. To this end, we use the Lipschitz continuity assumptions (\Cref{ass:lipschitz_continuity}) similar to \citet{curi2020efficient}. This results in comparing the deviation between two trajectories under the same policy and dynamics but different initial states (see \Cref{lem:optimistic_planning_regret}). For many systems (even linear) this can grow exponentially in the horizon $T$. Accordingly, we obtain a $\beta^T_N$ term in our bound (\Cref{eq:bound_gp}). Nonetheless, for cases where the RKHS is of a kernel with maximum information gain  $\gamma_N(k) 
= \setO\left(\text{polylog}(N)\right)$, we can give sample complexity bounds and an almost sure convergence result in the reachable set $\setR$ (\Cref{eq:reachabilty_set_convergence}).
Kernels such as the RBF kernel or the linear kernel (kernel with a finite-dimensional feature map $\phi(x)$) have maximum information gain which grows polylogarithmically with $n$ (\citet{vakili2021information}). Therefore, our convergence guarantees hold for a very rich class of functions.
The exponential dependence of our bound on $T$ imposes the restriction on the kernel class. For the case of Gaussian noise, we can include a richer class of kernels, such as Matèrn.

In addition to the convergence results above, we also give guarantees on the zero-shot performance of \opacex in \Cref{sec:zero_shot_performance_theory}.
\section{Experiments}\label{sec:experiments}
We evaluate \opacex on the Pendulum-v1 and MountainCar environment from the OpenAI gym benchmark suite~\citep{brockman2016openai}, on the Reacher, Swimmer, and Cheetah from the deep mind control suite~\citep{tassa2018deepmind}, and a high-dimensional simulated robotic manipulation task introduced by~\citet{li2020towards}. See \Cref{sec:experiment_details} for more details on the experimental setup. 
\paragraph{Baselines}
\looseness=-1
We implement four baselines for comparisons. To show the benefit of our intrinsic reward, we compare \opacex to (\emph{1}) a random exploration policy (\textsc{Random}) which randomly samples actions from the action space. As we discuss in \Cref{sec:opacex} our choice of objective in \Cref{eq:exploration_op} is in essence similar to the one proposed by~\cite{pathak2019self} and \cite{sekar2020planning}. Therefore, in our experiments, we compare the optimistic planner with other planning approaches.
Moreover, most work on active exploration either uses the mean planner or does not specify the planner (c.f., \Cref{sec:related_works}). We use the most common planners: (\emph{2}) mean (\textsc{Mean-AE}), and (\emph{3}) trajectory sampling (TS-1) scheme proposed in \cite{chua2018pets} (\textsc{PETS-AE}) as our baselines.
The mean planner simply uses the mean estimate $\vmu_n$  of the well-calibrated model. 
This is also used in \cite{active_learning_gp}.
Finally, we compare \opacex to (\emph{4}) H-UCRL~\citep{curi2020efficient}, a single-task model-based RL algorithm.
We investigate the following three aspects: (\emph{i}) {\em how fast does active exploration reduce model's epistemic uncertainty $\vsigma_n$ with increasing $n$}, (\emph{ii}) {\em can we solve downstream tasks with \opacex}, and (\emph{iii}) {\em does \opacex scale to high-dimensional and challenging object manipulation tasks}? For our experiments, we use GPs and probabilistic ensembles (PE,~\cite{lakshminarayanan2017simple}) for modeling the dynamics. For the planning, we either the soft actor-critic (SAC, \cite{sac}) policy optimizer, which takes simulated trajectories from our learned model to train a policy, or MPC with the iCEM optimizer~\citep{iCem}.
\begin{figure}[th]
    \centering
    \includegraphics[width=0.75\textwidth]{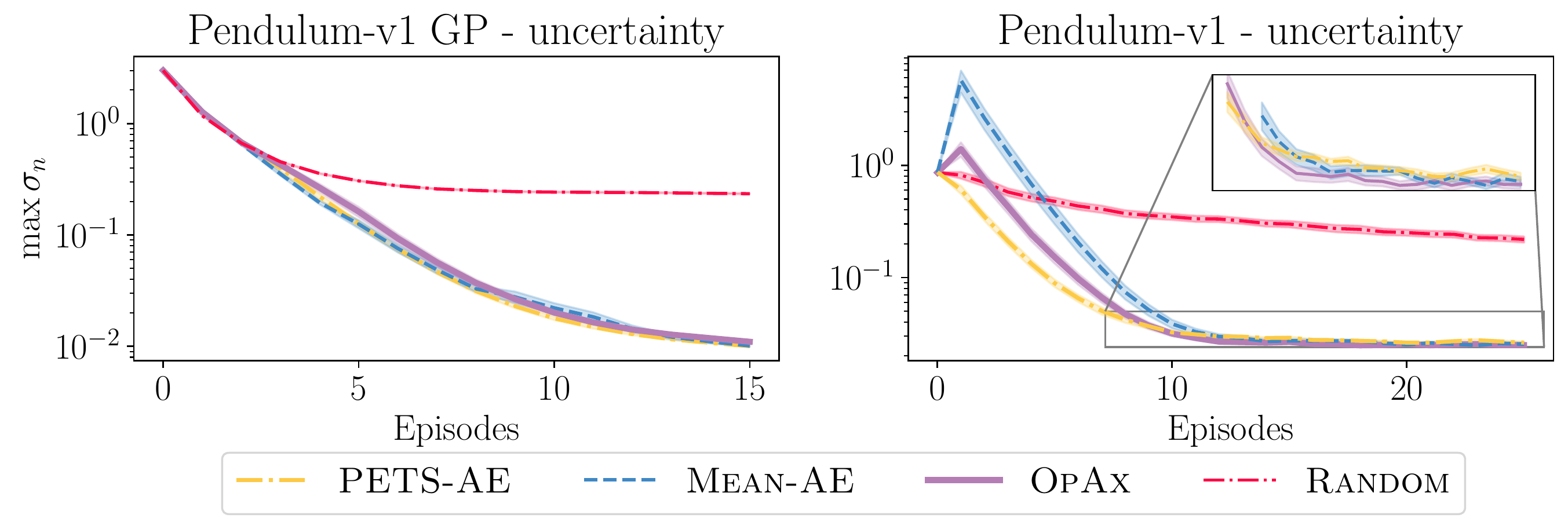}
    \caption{\looseness -1 Reduction in maximum epistemic uncertainty in \emph{reachable} state-action space for the Pendulum-v1 environment over 10 different random seeds. We evaluate \opacex with both GPs and PE and plot the mean performance with two standard error confidence intervals. For both, active exploration reduces epistemic uncertainty faster compared to random exploration. All active exploration baselines perform well for the GP case, whereas for the PE case \opacex gives slightly lower uncertainties.}
    \label{fig:epistemic_uncertainty_figure}
\end{figure}
\begin{figure}[ht]
    \centering
    \includegraphics[width=1\textwidth]{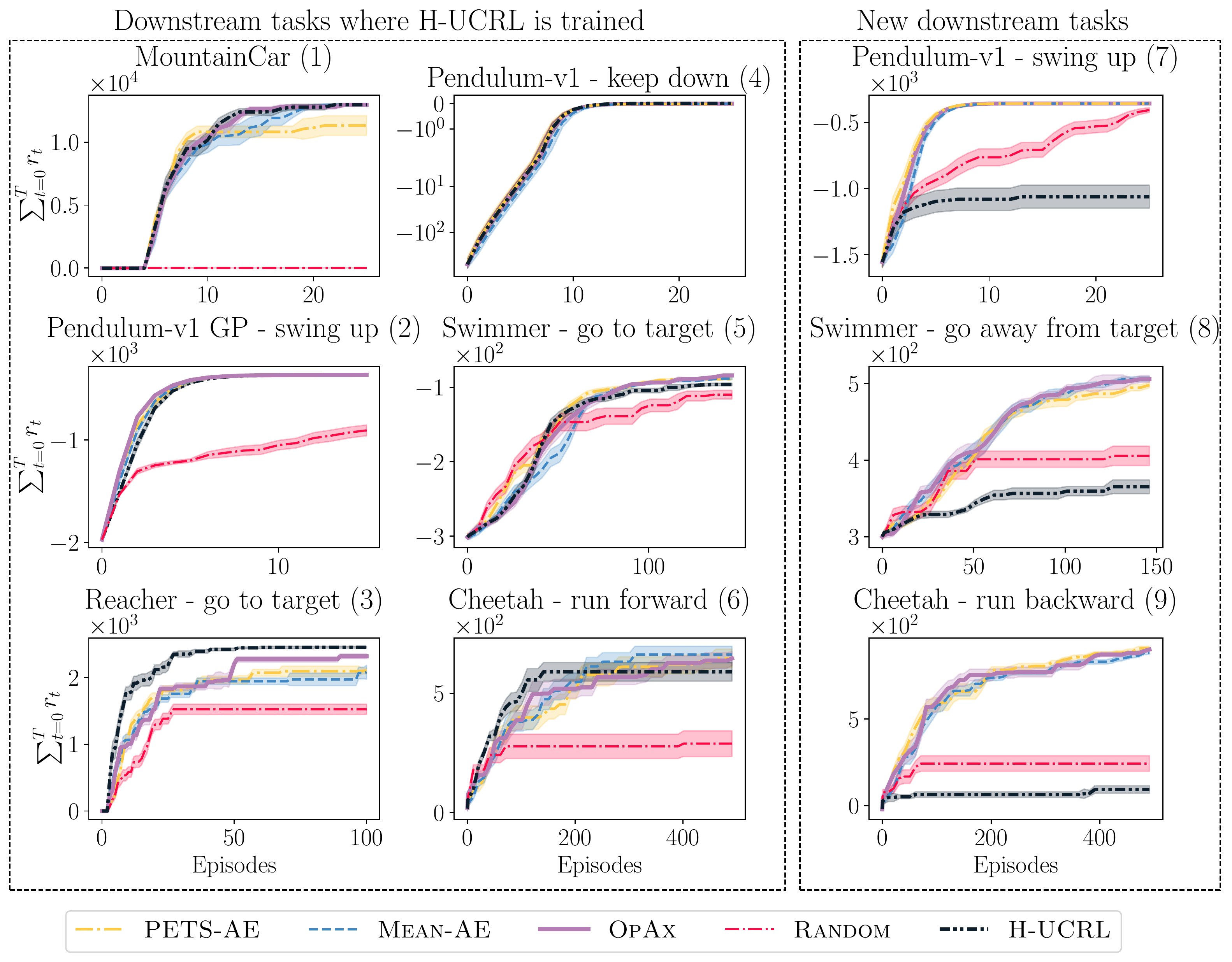}
    \caption{\capt{\looseness-1 We evaluate the downstream performance of our agents over 10 different random seeds and plot the mean performance with two standard error confidence intervals.
    For all the environments we use PE as models, except plot (1), for which we use a GP model (see plot (2) in the figure above). For tasks (1)-(6), we also train \textsc{H-UCRL}, a model-based RL algorithm. Tasks (7)-(9) are new/unseen for \textsc{H-UCRL}. From the Figure, we conclude that (\emph{i}) compared to other active exploration baselines, \opacex constantly performs well and is on par with \textsc{H-UCRL}, and (\emph{ii}) on the new/unseen tasks the active exploration baselines and \opacex outperform \textsc{H-UCRL} by a large margin.
    }}
    \label{fig:results_figure}
\end{figure}

\paragraph{How fast does active exploration reduce the epistemic uncertainty?}
For this experiment, we consider the Pendulum-v1 environment. We sample transitions at random from the pendulum's \emph{reachable} state-action space and evaluate our model's epistemic uncertainty for varying episodes and baselines. We model the dynamics with both GPs and PE. We depict the result
in \Cref{fig:epistemic_uncertainty_figure}. We conclude that the \textsc{Random} agent is slower in reducing the uncertainty compared to other active exploration methods for both GP and PE models. In particular, from the experiment, we empirically validate \Cref{thm:gp_theorem} for the GP case and also conclude that empirically even when using PE models, we find convergence of epistemic uncertainty. Moreover, we notice for the PE case 
that \opacex reaches smaller uncertainties slightly faster than \textsc{Mean-AE} and \textsc{PETS-AE}. We believe this is due to the additional exploration induced by the optimistic planner.

\paragraph{Can the model learnt through \opacex solve downstream tasks?}
\looseness=-1
We use \opacex and other active exploration baselines to actively learn a dynamics model and then evaluate the learned model on downstream tasks. We consider several tasks, (\emph{i}) Pendulum-v1 swing up, (\emph{ii})  Pendulum-v1 keep down (keep the pendulum at the stable equilibria), (\emph{iii}) MountainCar, (\emph{iv}) Reacher - go to target, (\emph{v}) Swimmer - go to target, 
(\emph{vi}) Swimmer - go away from target (quickly go away from the target position), 
(\emph{vii}) Cheetah - run forward, (\emph{viii}) Cheetah - run backward. For all tasks, we consider PEs, except for (\emph{i}) where we also use GPs. Furthermore,
for the MountainCar and Reacher, we give a reward once the goal is reached. Since this requires long-term planning, we use a SAC policy for these tasks. We use MPC with iCEM for the remaining tasks. We also train \textsc{H-UCRL} on tasks (\emph{i}) with GPs, and (\emph{ii}), (\emph{iii}), (\emph{iv}), (\emph{v}), (\emph{vii}) with PEs. We report the best performance across all episodes. 

To make a fair comparison, we use the following evaluation procedure; first, we perform active exploration for each episode on the environment, and then after every few episodes we use the mean estimate $\vmu_n$ to evaluate our learned model on the downstream tasks. 

\looseness=-1
\Cref{fig:results_figure} shows that all active exploration variants perform considerably better than the \textsc{Random} agent. In particular, for the MountainCar, the \textsc{Random} agent is not able to solve the task. Moreover, \textsc{PETS-AE} performs slightly worse than the other exploration baselines in this environment.
In general, we notice that \opacex always performs well and is able to achieve \textsc{H-UCRL}'s performance on all the tasks for which \textsc{H-UCRL} is trained. 
However, on tasks that are new/unseen for \textsc{H-UCRL}, active exploration algorithms outperform \textsc{H-UCRL}. 
From this experiment, we conclude two things 
(\emph{1}) apart from providing theoretical guarantees, the model learned through \opacex 
also performs well in downstream tasks, and (\emph{2}) active exploration agents generalize well to downstream tasks, whereas \textsc{H-UCRL} performs considerably worse on new/unseen tasks.
We believe this is because, unlike active exploration agents, task-specific model-based RL agents only explore the regions of the state-action space that are relevant to the task at hand.

\paragraph{Does \opacex scale to high-dimensional and challenging object manipulation tasks?}
\begin{wrapfigure}[13]{r}{0.2\textwidth}
    \label{fig:fpp_pic}
    \centering
    \includegraphics[width=\linewidth]{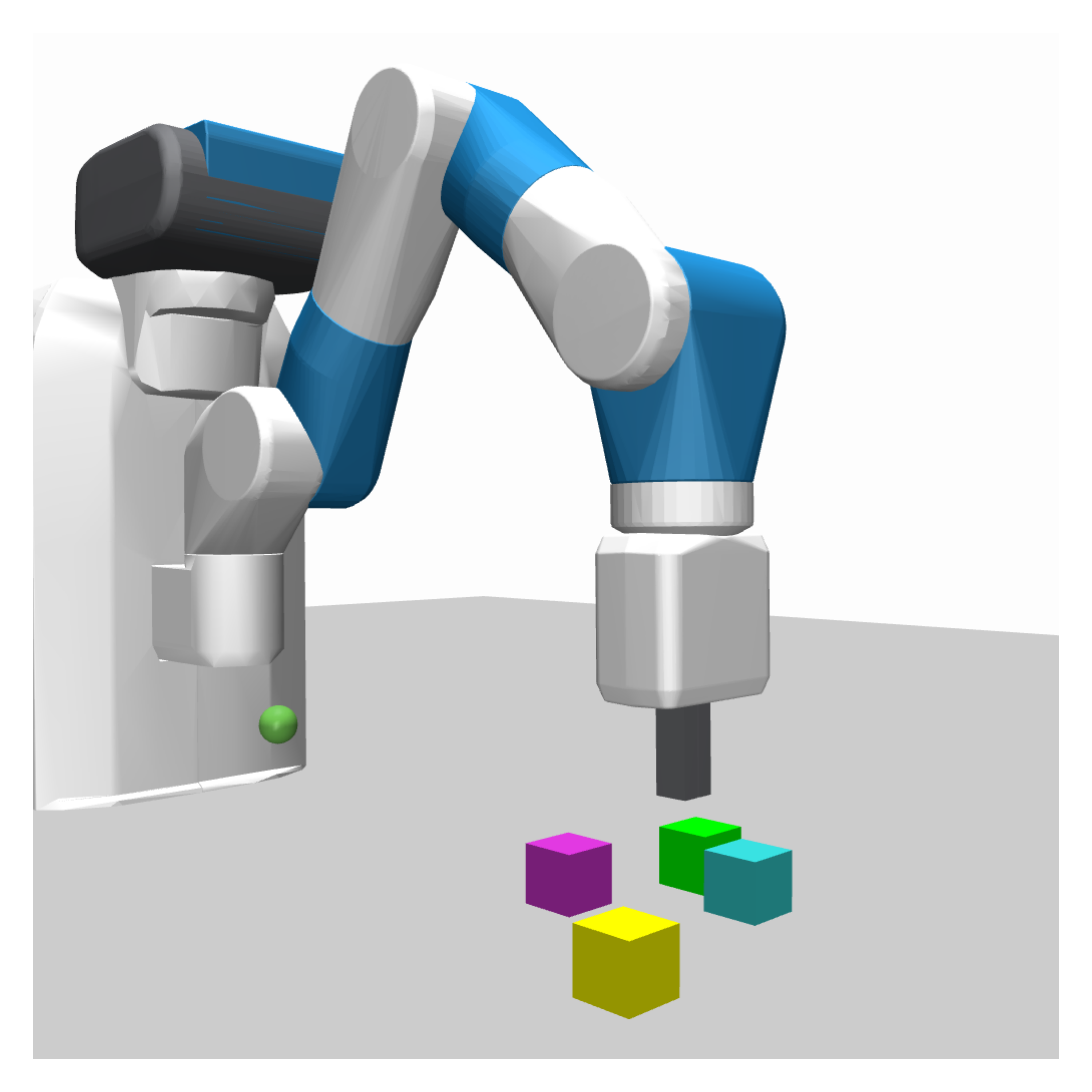}
    \caption{Fetch Pick \& Place Construction environment.}
\setlength{\intextsep}{1.0pt plus 2.0pt minus 2.0pt}
\setlength{\columnsep}{10pt}%
\end{wrapfigure}
To answer this question, we consider the Fetch Pick \& Place Construction environment~\citep{li2020towards}. We again use the active exploration agents to learn a model and then evaluate the success rate of the learned model in three challenging downstream tasks: (\emph{i}) Pick \& Place, (\emph{ii}) Throw, and (\emph{iii}) Flip (see \Cref{fig:construction:success_rates}). The environment contains a \( 7 \)-DoF robot arm and four \( 6 \)-DoF blocks that can be manipulated. In total, the state space is \( 58 \)-dimensional. The \( 4 \)-dimensional actions control the end-effector of the robot in Cartesian space as well as the opening/closing of the gripper. We compare \opacex to \textsc{PETS-AE}, \textsc{Mean-AE}, a random policy as well as \ceeus~\citep{Sancaktaretal22}. \ceeus is a model-based active exploration algorithm, for which \cite{Sancaktaretal22} reports state-of-the-art performance compared to several other active exploration methods. In all three tasks, \opacex is at least on par with the best-performing baselines, including \ceeus.
We run \opacex and all baselines with the same architecture and hyperparameter settings. See \Cref{sec:experiment_details} for more details. \looseness -1
\begin{figure}[ht]
    \centering
    \begin{subfigure}{.32\linewidth}
        \includegraphics[width=\linewidth]{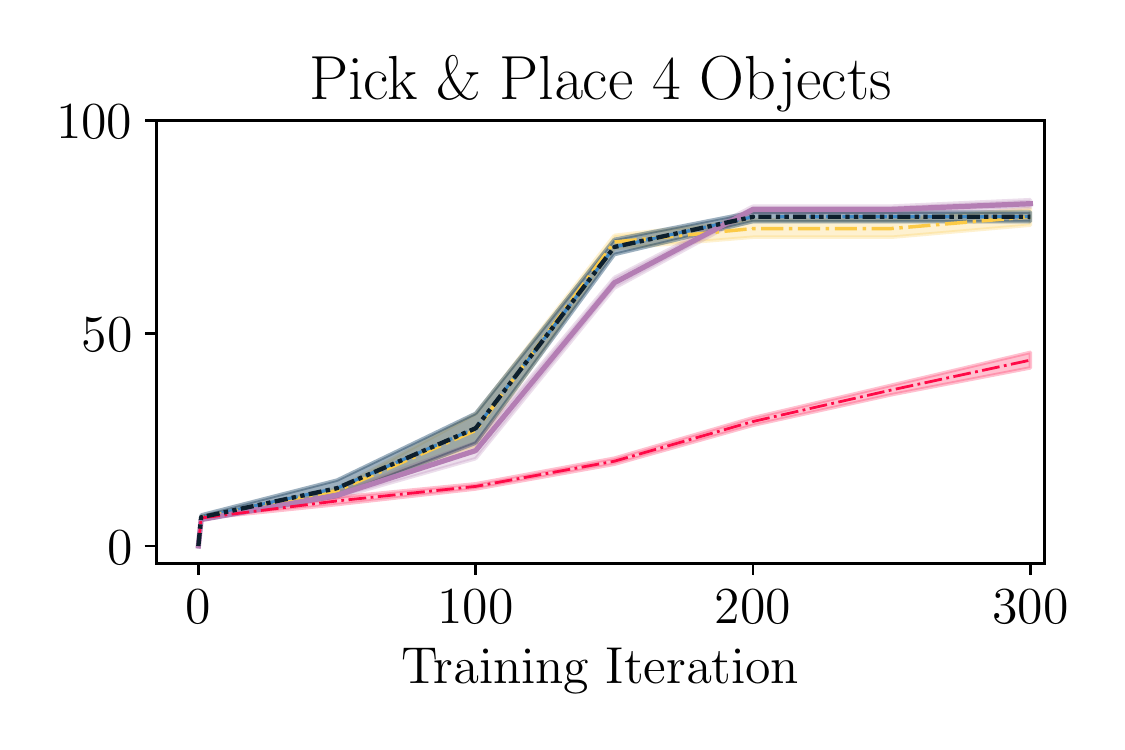}
    \end{subfigure}
    \begin{subfigure}{.32\linewidth}
        \includegraphics[width=\linewidth]{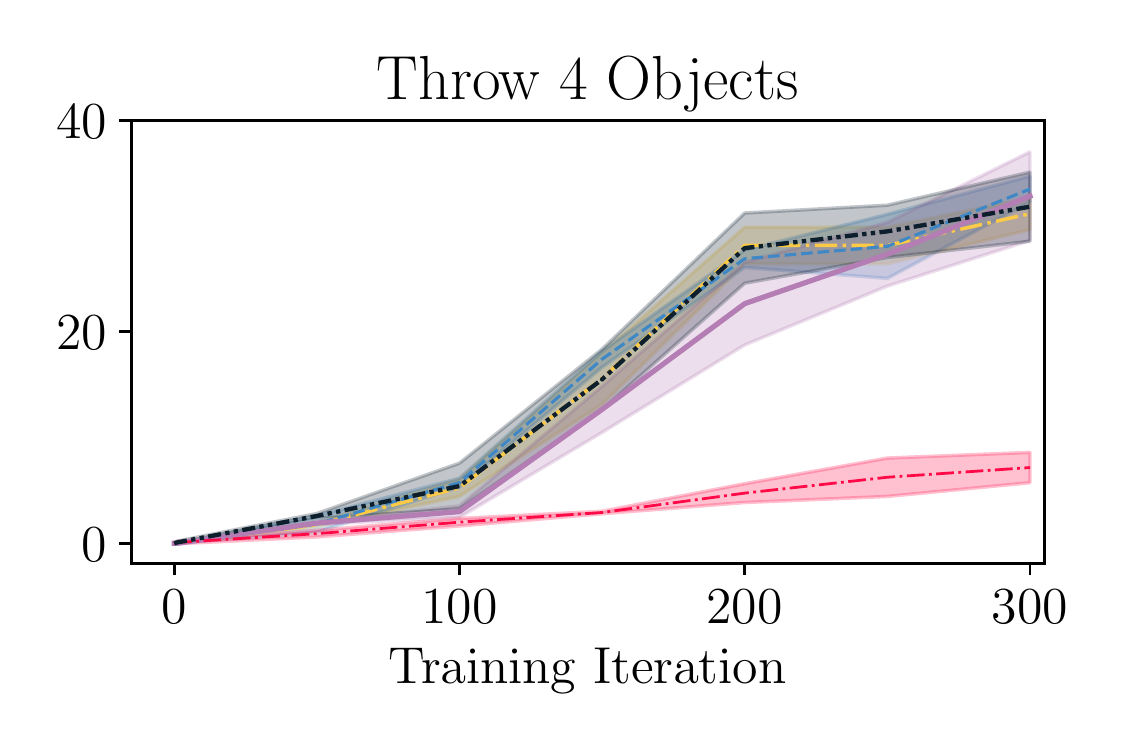}
    \end{subfigure}
    \begin{subfigure}{.32\linewidth}
        \includegraphics[width=\linewidth]{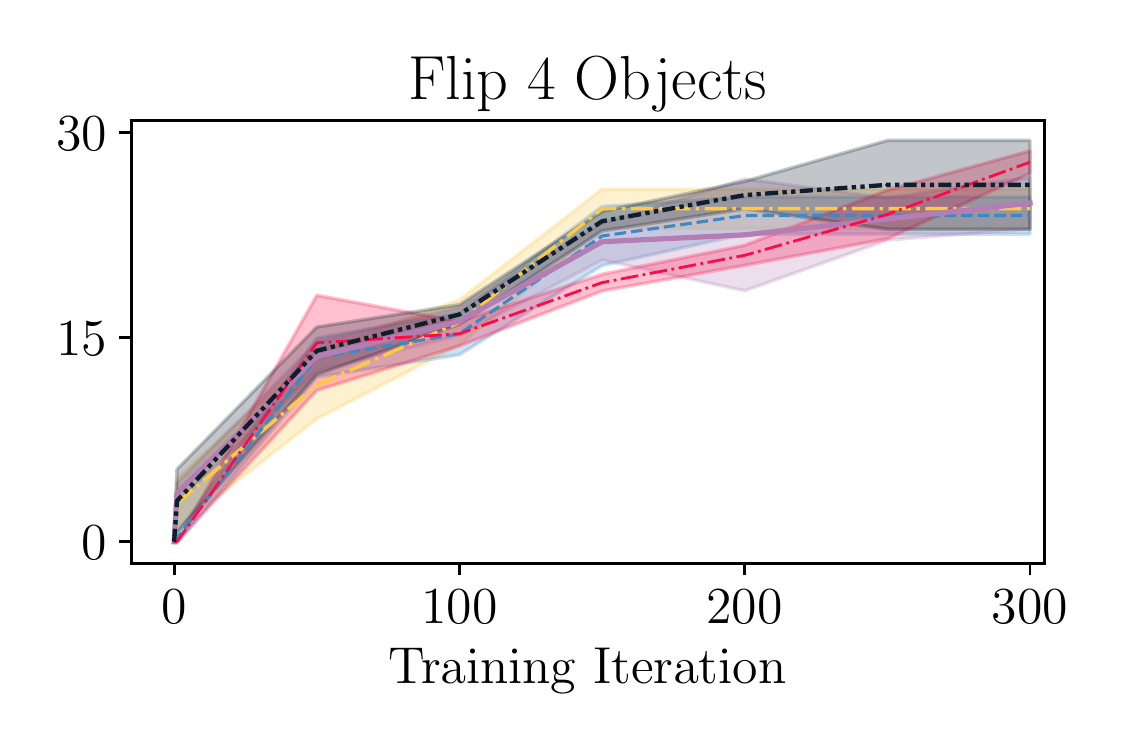}
    \end{subfigure}\\
    \includegraphics[width=.8\linewidth]{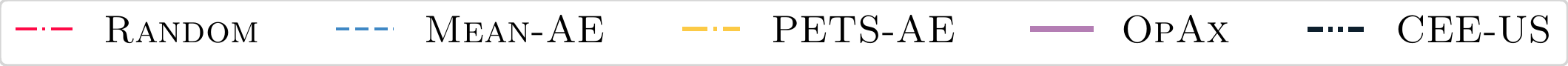}
    \caption{\looseness-1 Success rates for pick \& place, throwing and flipping tasks with four objects in the Fetch Pick \& Place Construction environment for \opacex and baselines. \capt{We evaluate task performance via planning zero-shot with models learned using different exploration strategies. We report performance on three independent seeds. \opacex is on par with the best-performing baselines in all tasks.}}
    \label{fig:construction:success_rates}
\end{figure}
\section{Related Work} \label{sec:related_works}
System identification is a broadly studied topic~\citep {ASTROMSystemID, schoukens, SCHON201139, ziemann2022single, ziemann2022learning}. However, system identification from the perspective of experiment design for nonlinear systems is much less understood~\citep{chiuso2019system}. Most methods formulate the identification task through the maximization of intrinsic rewards.
Common choices of intrinsic rewards are (\emph{i}) model prediction error or ``Curiosity''~\citep{schmidhuber1991possibility, pathak2017curiosity}, (\emph{ii}) novelty of transitions~\citep{stadie2015incentivizing}, and  (\emph{iii}) diversity of skills~\citep{eysenbach2018diversity}.

A popular choice for intrinsic rewards is mutual information or entropy~\citep{learning_and_control_using_gps, active_learning_gp, max, pathak2019self, sekar2020planning}. \cite{learning_and_control_using_gps} propose an
approach to maximize the information gain greedily wrt the immediate next transition, i.e., one-step greedy, whereas \cite{active_learning_gp} consider planning full trajectories.
\cite{max, pathak2019self, sekar2020planning} and \cite{Sancaktaretal22} consider general Bayesian models, such as BNNs, to represent a probabilistic distribution for the learned model. 
\cite{max} propose using the information gain of the model with respect to observed transition as the intrinsic reward. 
To this end, they learn an ensemble of Gaussian neural networks and represent the distribution over models with a Gaussian mixture model (GMM). A similar approach is also proposed in \cite{pathak2019self, sekar2020planning, Sancaktaretal22}. The main difference between \cite{max} and \cite{pathak2019self} lies in how they represent mutual information. Moreover, \cite{pathak2019self} use the model's epistemic uncertainty, that is the disagreement between the ensemble models as an intrinsic reward. \cite{sekar2020planning} link the model disagreement (epistemic uncertainty) reward to maximizing mutual information and demonstrate state-of-the-art performance on several high-dimensional tasks. 
Similarly, \cite{Sancaktaretal22}, use the disagreement in predicted trajectories of an ensemble of neural networks to direct exploration. Since trajectories can diverge due to many factors beyond just the model epistemic uncertainty, e.g., aleatoric noise, this approach is restricted to deterministic systems and susceptible to systems with unstable equilibria. 
Our approach is the most similar to \cite{pathak2019self, sekar2020planning} since we also propose the model epistemic uncertainty as the intrinsic reward for planning. However, we thoroughly and theoretically motivate this choice of reward from a Bayesian experiment design perspective. Furthermore,
we induce additional exploration in \opacex through our optimistic planner and rigorously study the theoretical properties of the proposed methods. On the contrary, most of the prior work discussed above either uses the mean planner (\textsc{Mean-AE}) or does not discuss the planner thoroughly or provide any theoretical results. In general, theoretical guarantees for active exploration algorithms are rather immature~\citep{chakraborty2023steering, wagenmaker2023optimal} and mostly restrictive to a small class of systems~\citep{simchowitz2018learning, tarbouriech2020active, wagenmaker2020active, mania2020active}. 
To the best of our knowledge, we are the first to give convergence guarantees for a rich class of nonlinear systems.

\looseness=-1
While our work focuses on the active learning of dynamics, there are numerous works that study exploration in the context of reward-free RL \citep{jin2020reward, kaufmann2021adaptive, wagenmaker2022reward, chen2022statistical}. However, most methods in this setting give guarantees for special classes of MDPs~\citep{jin2020reward, kaufmann2021adaptive, wagenmaker2022reward, qiu2021reward, chen2022statistical} and result in practical algorithms. On the contrary, we focus on solely learning the dynamics. While a good dynamics model may be used for zero-shot planning, it also exhibits more relevant knowledge about the system such as its stability or sensitivity to external effects. Furthermore, our proposed method is not only theoretically sound but also practical. 
\section{Conclusion}
\looseness=-1
We present \opacex, a novel model-based RL algorithm for the active exploration of unknown dynamical systems. Taking inspiration from Bayesian experiment design, we provide a comprehensive explanation for using model epistemic uncertainty as an intrinsic reward for exploration. 
By leveraging the \emph{optimistic in the face of uncertainty} paradigm, we put forth first-of-their-kind theoretical results on the convergence of active exploration agents in reinforcement learning. Specifically, we study convergence properties of general Bayesian models, such as BNNs. For the frequentist case of RKHS dynamics, we established sample complexity bounds and convergence guarantees for \opacex for a rich class of functions.
 We evaluate the efficacy of \opacex across various RL environments with state space dimensions from two to 58. The empirical results corroborate our theoretical findings, as \opacex displays systematic and effective exploration across all tested environments and exhibits strong performance in zero-shot planning for new downstream tasks.
\clearpage
\begin{ack}
We would like to thank Jonas H\"ubotter for the insightful discussions and his feedback on this work. Furthermore, we also thank Alex Hägele, 
Parnian Kassraie, Scott Sussex, and Dominik Baumann for their feedback. 

This project has received funding from the Swiss National Science Foundation under NCCR Automation, grant agreement 51NF40 180545, and the Microsoft Swiss Joint Research Center.

\end{ack}


\bibliographystyle{apalike}
\bibliography{refs}

\begin{thebibliography}{}

\bibitem[Balcan et~al., 2010]{balcan2010true}
Balcan, M.-F., Hanneke, S., and Vaughan, J.~W. (2010).
\newblock The true sample complexity of active learning.
\newblock {\em Machine learning}, 80:111--139.

\bibitem[Berkenkamp et~al., 2017]{berkenkamp2017safe}
Berkenkamp, F., Turchetta, M., Schoellig, A., and Krause, A. (2017).
\newblock Safe model-based reinforcement learning with stability guarantees.
\newblock {\em NeurIPS}, 30.

\bibitem[Biagiotti and Melchiorri, 2008]{trajectorPlanningForRobots}
Biagiotti, L. and Melchiorri, C. (2008).
\newblock {\em Trajectory Planning for Automatic Machines and Robots}.
\newblock Springer Publishing Company, Incorporated, 1st edition.

\bibitem[Bradbury et~al., 2018]{jax2018github}
Bradbury, J., Frostig, R., Hawkins, P., Johnson, M.~J., Leary, C., Maclaurin, D., Necula, G., Paszke, A., Vander{P}las, J., Wanderman-{M}ilne, S., and Zhang, Q. (2018).
\newblock {JAX}: composable transformations of {P}ython+{N}um{P}y programs.

\bibitem[Brockman et~al., 2016]{brockman2016openai}
Brockman, G., Cheung, V., Pettersson, L., Schneider, J., Schulman, J., Tang, J., and Zaremba, W. (2016).
\newblock Openai gym.
\newblock {\em arXiv preprint arXiv:1606.01540}.

\bibitem[Buisson-Fenet et~al., 2020]{active_learning_gp}
Buisson-Fenet, M., Solowjow, F., and Trimpe, S. (2020).
\newblock Actively learning gaussian process dynamics.
\newblock In Bayen, A.~M., Jadbabaie, A., Pappas, G., Parrilo, P.~A., Recht, B., Tomlin, C., and Zeilinger, M., editors, {\em L4DC}.

\bibitem[Chakraborty et~al., 2023]{chakraborty2023steering}
Chakraborty, S., Bedi, A., Koppel, A., Wang, M., Huang, F., and Manocha, D. (2023).
\newblock {STEERING} : Stein information directed exploration for model-based reinforcement learning.
\newblock pages 3949--3978.

\bibitem[Chaloner and Verdinelli, 1995]{chaloner1995bayesian}
Chaloner, K. and Verdinelli, I. (1995).
\newblock Bayesian experimental design: A review.
\newblock {\em Statistical science}, pages 273--304.

\bibitem[Chen et~al., 2022]{chen2022statistical}
Chen, J., Modi, A., Krishnamurthy, A., Jiang, N., and Agarwal, A. (2022).
\newblock On the statistical efficiency of reward-free exploration in non-linear rl.
\newblock {\em Advances in Neural Information Processing Systems}, 35:20960--20973.

\bibitem[Chen et~al., 2015]{chen15mis}
Chen, Y., Hassani, S.~H., Karbasi, A., and Krause, A. (2015).
\newblock Sequential information maximization: When is greedy near-optimal?
\newblock In {\em COLT}.

\bibitem[Chiuso and Pillonetto, 2019]{chiuso2019system}
Chiuso, A. and Pillonetto, G. (2019).
\newblock System identification: A machine learning perspective.
\newblock {\em Annual Review of Control, Robotics, and Autonomous Systems}, 2:281--304.

\bibitem[Chowdhury and Gopalan, 2017]{chowdhury2017kernelized}
Chowdhury, S.~R. and Gopalan, A. (2017).
\newblock On kernelized multi-armed bandits.
\newblock In {\em ICML}, pages 844--853. PMLR.

\bibitem[Chua et~al., 2018]{chua2018pets}
Chua, K., Calandra, R., McAllister, R., and Levine, S. (2018).
\newblock Deep reinforcement learning in a handful of trials using probabilistic dynamics models.
\newblock In {\em NeurIPS}.

\bibitem[Cover and Thomas, 2006]{elementsofIT}
Cover, T.~M. and Thomas, J.~A. (2006).
\newblock {\em Elements of Information Theory (Wiley Series in Telecommunications and Signal Processing)}.
\newblock Wiley-Interscience.

\bibitem[Curi et~al., 2020]{curi2020efficient}
Curi, S., Berkenkamp, F., and Krause, A. (2020).
\newblock Efficient model-based reinforcement learning through optimistic policy search and planning.
\newblock {\em NeurIPS}, 33:14156--14170.

\bibitem[Eysenbach et~al., 2018]{eysenbach2018diversity}
Eysenbach, B., Gupta, A., Ibarz, J., and Levine, S. (2018).
\newblock Diversity is all you need: Learning skills without a reward function.
\newblock {\em arXiv preprint arXiv:1802.06070}.

\bibitem[Garc\'ia et~al., 1989]{mpc}
Garc\'ia, C.~E., Prett, D.~M., and Morari, M. (1989).
\newblock Model predictive control: Theory and practice - a survey.
\newblock {\em Automatica}, pages 335--348.

\bibitem[Haarnoja et~al., 2018]{sac}
Haarnoja, T., Zhou, A., Abbeel, P., and Levine, S. (2018).
\newblock Soft actor-critic: Off-policy maximum entropy deep reinforcement learning with a stochastic actor.
\newblock In {\em ICML}, pages 1861--1870.

\bibitem[Hanneke et~al., 2014]{hanneke2014theory}
Hanneke, S. et~al. (2014).
\newblock Theory of disagreement-based active learning.
\newblock {\em Foundations and Trends{\textregistered} in Machine Learning}, 7(2-3):131--309.

\bibitem[Jain et~al., 2018]{learning_and_control_using_gps}
Jain, A., Nghiem, T., Morari, M., and Mangharam, R. (2018).
\newblock Learning and control using gaussian processes.
\newblock In {\em ICCPS}, pages 140--149.

\bibitem[Jin et~al., 2020]{jin2020reward}
Jin, C., Krishnamurthy, A., Simchowitz, M., and Yu, T. (2020).
\newblock Reward-free exploration for reinforcement learning.
\newblock In {\em ICML}, pages 4870--4879. PMLR.

\bibitem[Kakade et~al., 2020]{kakade2020information}
Kakade, S., Krishnamurthy, A., Lowrey, K., Ohnishi, M., and Sun, W. (2020).
\newblock Information theoretic regret bounds for online nonlinear control.
\newblock {\em NeurIPS}, 33:15312--15325.

\bibitem[Kaufmann et~al., 2021]{kaufmann2021adaptive}
Kaufmann, E., M{\'e}nard, P., Domingues, O.~D., Jonsson, A., Leurent, E., and Valko, M. (2021).
\newblock Adaptive reward-free exploration.
\newblock In {\em ALT}, pages 865--891. PMLR.

\bibitem[Khalil, 2015]{khalil2015nonlinear}
Khalil, H.~K. (2015).
\newblock {\em Nonlinear control}, volume 406.
\newblock Pearson New York.

\bibitem[Krause et~al., 2008]{krause2008near}
Krause, A., Singh, A., and Guestrin, C. (2008).
\newblock Near-optimal sensor placements in gaussian processes: Theory, efficient algorithms and empirical studies.
\newblock {\em Journal of Machine Learning Research}, 9(2).

\bibitem[Kuleshov et~al., 2018]{kuleshov2018accurate}
Kuleshov, V., Fenner, N., and Ermon, S. (2018).
\newblock Accurate uncertainties for deep learning using calibrated regression.
\newblock In {\em ICML}, pages 2796--2804. PMLR.

\bibitem[Lakshminarayanan et~al., 2017]{lakshminarayanan2017simple}
Lakshminarayanan, B., Pritzel, A., and Blundell, C. (2017).
\newblock Simple and scalable predictive uncertainty estimation using deep ensembles.
\newblock {\em NeurIPS}, 30.

\bibitem[Li et~al., 2020]{li2020towards}
Li, R., Jabri, A., Darrell, T., and Agrawal, P. (2020).
\newblock Towards practical multi-object manipulation using relational reinforcement learning.
\newblock In {\em ICRA}, pages 4051--4058. IEEE.

\bibitem[Lindley, 1956]{lidley_info_gain}
Lindley, D.~V. (1956).
\newblock {On a Measure of the Information Provided by an Experiment}.
\newblock {\em The Annals of Mathematical Statistics}, pages 986 -- 1005.

\bibitem[Mania et~al., 2020]{mania2020active}
Mania, H., Jordan, M.~I., and Recht, B. (2020).
\newblock Active learning for nonlinear system identification with guarantees.
\newblock {\em arXiv preprint arXiv:2006.10277}.

\bibitem[Mehta et~al., 2021]{mehta2021experimental}
Mehta, V., Paria, B., Schneider, J., Ermon, S., and Neiswanger, W. (2021).
\newblock An experimental design perspective on model-based reinforcement learning.
\newblock {\em arXiv preprint arXiv:2112.05244}.

\bibitem[Mutny et~al., 2023]{pmlr-v206-mutny23a}
Mutny, M., Janik, T., and Krause, A. (2023).
\newblock Active exploration via experiment design in markov chains.
\newblock In {\em AISTATS}.

\bibitem[Pasztor et~al., 2021]{pasztor2021efficient}
Pasztor, B., Bogunovic, I., and Krause, A. (2021).
\newblock Efficient model-based multi-agent mean-field reinforcement learning.
\newblock {\em arXiv preprint arXiv:2107.04050}.

\bibitem[Pathak et~al., 2017]{pathak2017curiosity}
Pathak, D., Agrawal, P., Efros, A.~A., and Darrell, T. (2017).
\newblock Curiosity-driven exploration by self-supervised prediction.
\newblock In {\em ICML}, pages 2778--2787. PMLR.

\bibitem[Pathak et~al., 2019]{pathak2019self}
Pathak, D., Gandhi, D., and Gupta, A. (2019).
\newblock Self-supervised exploration via disagreement.
\newblock In {\em ICML}, pages 5062--5071. PMLR.

\bibitem[Pinneri et~al., 2021]{iCem}
Pinneri, C., Sawant, S., Blaes, S., Achterhold, J., Stueckler, J., Rolinek, M., and Martius, G. (2021).
\newblock Sample-efficient cross-entropy method for real-time planning.
\newblock In {\em CORL}, Proceedings of Machine Learning Research, pages 1049--1065.

\bibitem[Qiu et~al., 2021]{qiu2021reward}
Qiu, S., Ye, J., Wang, Z., and Yang, Z. (2021).
\newblock On reward-free rl with kernel and neural function approximations: Single-agent mdp and markov game.
\newblock In {\em ICML}, pages 8737--8747. PMLR.

\bibitem[Rasmussen and Williams, 2005]{rasmussen2005gp}
Rasmussen, C.~E. and Williams, C. K.~I. (2005).
\newblock {\em Gaussian Processes for Machine Learning (Adaptive Computation and Machine Learning)}.
\newblock The MIT Press.

\bibitem[\r{A}str\"{o}m and Eykhoff, 1971]{ASTROMSystemID}
\r{A}str\"{o}m, K. and Eykhoff, P. (1971).
\newblock System identification - a survey.
\newblock {\em Automatica}, 7(2):123--162.

\bibitem[Rothfuss et~al., 2023]{rothfuss2023hallucinated}
Rothfuss, J., Sukhija, B., Birchler, T., Kassraie, P., and Krause, A. (2023).
\newblock Hallucinated adversarial control for conservative offline policy evaluation.
\newblock {\em UAI}.

\bibitem[Sancaktar et~al., 2022]{Sancaktaretal22}
Sancaktar, C., Blaes, S., and Martius, G. (2022).
\newblock Curious exploration via structured world models yields zero-shot object manipulation.
\newblock In {\em NeurIPS 35 (NeurIPS 2022)}.

\bibitem[Schmidhuber, 1991]{schmidhuber1991possibility}
Schmidhuber, J. (1991).
\newblock A possibility for implementing curiosity and boredom in model-building neural controllers.
\newblock In {\em Proc. of the international conference on simulation of adaptive behavior: From animals to animats}, pages 222--227.

\bibitem[Sch\"on et~al., 2011]{SCHON201139}
Sch\"on, T.~B., Wills, A., and Ninness, B. (2011).
\newblock System identification of nonlinear state-space models.
\newblock {\em Automatica}, pages 39--49.

\bibitem[Schoukens and Ljung, 2019]{schoukens}
Schoukens, J. and Ljung, L. (2019).
\newblock Nonlinear system identification: A user-oriented road map.
\newblock {\em IEEE Control Systems Magazine}, 39(6):28--99.

\bibitem[Sekar et~al., 2020]{sekar2020planning}
Sekar, R., Rybkin, O., Daniilidis, K., Abbeel, P., Hafner, D., and Pathak, D. (2020).
\newblock Planning to explore via self-supervised world models.
\newblock In {\em ICML}, pages 8583--8592. PMLR.

\bibitem[Settles, 2009]{settles2009active}
Settles, B. (2009).
\newblock Active learning literature survey.

\bibitem[Shyam et~al., 2019]{max}
Shyam, P., Ja{\'s}kowski, W., and Gomez, F. (2019).
\newblock Model-based active exploration.
\newblock In {\em ICML}, pages 5779--5788. PMLR.

\bibitem[Simchowitz and Foster, 2020]{linearExploration}
Simchowitz, M. and Foster, D. (2020).
\newblock Naive exploration is optimal for online {LQR}.
\newblock In {\em ICML}, Proceedings of Machine Learning Research, pages 8937--8948. PMLR.

\bibitem[Simchowitz et~al., 2018]{simchowitz2018learning}
Simchowitz, M., Mania, H., Tu, S., Jordan, M.~I., and Recht, B. (2018).
\newblock Learning without mixing: Towards a sharp analysis of linear system identification.
\newblock In {\em COLT}, pages 439--473. PMLR.

\bibitem[Srinivas et~al., 2012]{srinivas}
Srinivas, N., Krause, A., Kakade, S.~M., and Seeger, M.~W. (2012).
\newblock Information-theoretic regret bounds for gaussian process optimization in the bandit setting.
\newblock {\em IEEE Transactions on Information Theory}.

\bibitem[Stadie et~al., 2015]{stadie2015incentivizing}
Stadie, B.~C., Levine, S., and Abbeel, P. (2015).
\newblock Incentivizing exploration in reinforcement learning with deep predictive models.
\newblock {\em arXiv preprint arXiv:1507.00814}.

\bibitem[Sussex et~al., 2023]{sussex2022model}
Sussex, S., Makarova, A., and Krause, A. (2023).
\newblock Model-based causal bayesian optimization.
\newblock In {\em ICLR}.

\bibitem[Tarbouriech et~al., 2020]{tarbouriech2020active}
Tarbouriech, J., Shekhar, S., Pirotta, M., Ghavamzadeh, M., and Lazaric, A. (2020).
\newblock Active model estimation in markov decision processes.
\newblock In {\em UAI}, pages 1019--1028. PMLR.

\bibitem[Tassa et~al., 2018]{tassa2018deepmind}
Tassa, Y., Doron, Y., Muldal, A., Erez, T., Li, Y., Casas, D. d.~L., Budden, D., Abdolmaleki, A., Merel, J., Lefrancq, A., et~al. (2018).
\newblock Deepmind control suite.
\newblock {\em arXiv preprint arXiv:1801.00690}.

\bibitem[Vakili et~al., 2021]{vakili2021information}
Vakili, S., Khezeli, K., and Picheny, V. (2021).
\newblock On information gain and regret bounds in gaussian process bandits.
\newblock In {\em AISTATS}.

\bibitem[Wagenmaker and Jamieson, 2020]{wagenmaker2020active}
Wagenmaker, A. and Jamieson, K. (2020).
\newblock Active learning for identification of linear dynamical systems.
\newblock In {\em COLT}, pages 3487--3582. PMLR.

\bibitem[Wagenmaker et~al., 2023]{wagenmaker2023optimal}
Wagenmaker, A., Shi, G., and Jamieson, K. (2023).
\newblock Optimal exploration for model-based rl in nonlinear systems.
\newblock {\em arXiv preprint arXiv:2306.09210}.

\bibitem[Wagenmaker et~al., 2022]{wagenmaker2022reward}
Wagenmaker, A.~J., Chen, Y., Simchowitz, M., Du, S., and Jamieson, K. (2022).
\newblock Reward-free rl is no harder than reward-aware rl in linear markov decision processes.
\newblock In {\em ICML}, pages 22430--22456. PMLR.

\bibitem[Wang and Yeung, 2020]{bnnsurvey}
Wang, H. and Yeung, D.-Y. (2020).
\newblock A survey on bayesian deep learning.
\newblock {\em ACM computing surveys (csur)}, pages 1--37.

\bibitem[Weiss et~al., 2016]{weiss2016survey}
Weiss, K., Khoshgoftaar, T.~M., and Wang, D. (2016).
\newblock A survey of transfer learning.
\newblock {\em Journal of Big data}, pages 1--40.

\bibitem[Zhang and Yang, 2021]{zhang2021survey}
Zhang, Y. and Yang, Q. (2021).
\newblock A survey on multi-task learning.
\newblock {\em IEEE Transactions on Knowledge and Data Engineering}, pages 5586--5609.

\bibitem[Ziemann and Tu, 2022]{ziemann2022learning}
Ziemann, I. and Tu, S. (2022).
\newblock Learning with little mixing.
\newblock {\em NeurIPS}, 35:4626--4637.

\bibitem[Ziemann et~al., 2022]{ziemann2022single}
Ziemann, I.~M., Sandberg, H., and Matni, N. (2022).
\newblock Single trajectory nonparametric learning of nonlinear dynamics.
\newblock In {\em COLT}, pages 3333--3364. PMLR.

\end{thebibliography}

\newpage
\appendix

\renewcommand*\contentsname{Contents of Appendix}
\etocdepthtag.toc{mtappendix}
\etocsettagdepth{mtchapter}{none}
\etocsettagdepth{mtappendix}{subsection}
\tableofcontents

\newpage
\section{Proofs for \cref{sec:theoretical_results}} \label{sec:proofs}
We first prove some key properties of our active exploration objective in \Cref{eq:exploration_op}. Then, we prove \Cref{thm:main_theorem} which holds for general Bayesian models, and finally we prove \Cref{thm:gp_theorem}, which guarantees convergence for the frequentist setting where the dynamics are modeled using a GP. 
\begin{lemma}[Properties of \opacex's objective]
\label{cor:bounded_objective}
Let \Cref{assumption: Well Calibration Assumption} and \ref{ass:lipschitz_continuity} hold, then the following is true for all $n\geq 0$,
\begin{align*}
     \log \left(1 + \frac{\sigma_{n-1, j}^2(\vx_t, \vpi(\vx_t))}{\sigma^2}\right) &\ge 0 \tag{1} \\
     \frac{1}{2}\sup_{\vpi \in \Pi, \veta \in \Xi} \hspace{0.2em} \E\left[\left(\sum_{t=0}^{T-1} \sum^{d_x}_{j=1} \log \left(1 + \frac{\sigma_{n-1, j}^2(\vx_t, \vpi(\vx_t))}{\sigma^2}\right)\right)^2\right] &\leq \frac{1}{2} T^2 d^2_x \log^2 \left(1 + \frac{\sigma^2_{\max}}{\sigma^2}\right) \tag{2}, \\
      \left\vert \frac{1}{2}\sum^{d_x}_{j=1} \log \left(1 + \frac{\sigma_{n, j}^2(\vz)}{\sigma^{2}}\right) - \log \left(1 + \frac{\sigma_{n, j}^2(\vz')}{\sigma^{2}}\right) \right\vert &\leq \frac{d_x\sigma_{\max} L_{\vsigma}}{\sigma^2} \norm{\vz - \vz'} \tag{3}.
\end{align*}
where $\sigma_{\max} = \sup_{\vz \in \setZ; i \geq 0; 1\leq j\leq d_x} \sigma_{i, j}(\vz)$.
\end{lemma}
\begin{proof}
    The positivity of the reward follows from the positive definiteness of the epistemic uncertainty $\sigma_{n-1, j}$. For (2), the following holds
    \begin{align*}
        &\E\left[\left(\sum_{t=0}^{T-1} \sum^{d_x}_{j=1} \log \left(1 + \frac{\sigma_{n-1, j}^2(\vx_t, \vpi(\vx_t))}{\sigma^2}\right)\right)^2\right] \\
        &\leq  \E\left[\sum_{t=0}^{T-1} \sum^{d_x}_{j=1} T d_x \log^2 \left(1 + \frac{\sigma_{n-1, j}^2(\vx_t, \vpi(\vx_t))}{\sigma^2}\right)\right] \\
        &\leq \E\left[\left(\sum_{t=0}^{T-1} \sum^{d_x}_{j=1} T d_x \log^2 \left(1 + \frac{\sigma^2_{\max}}{\sigma^2}\right)\right)\right] \\
        &\leq T^2 d^2_x \log^2 \left(1 + \frac{\sigma^2_{\max}}{\sigma^2}\right)
    \end{align*}
    From hereon, let $J_{\max} = \sfrac{1}{2} T^2 d^2_x \log^2 \left(1 + \sigma^{-2}\sigma^2_{\max}\right)$.
    
      Finally, we show that this reward is Lipschitz continuous. 
    \begin{align*}
        \lvert\sigma^2_{n, j}(\vz)  -  \sigma^2_{n, j}(\vz')\rvert &= 
         \vert\sigma_{n, j}(\vz)  \sigma_{n, j}(\vz) -  \sigma_{n, j}(\vz)\sigma_{n, j}(\vz') + \sigma_{n, j}(\vz)\sigma_{n, j}(\vz') - \sigma_{n, j}(\vz')  \sigma_{n, j}(\vz')\rvert \\
         &\leq L_{\sigma}\norm{\vz - \vz'} \sigma_{n, j}(\vz) +  L_{\sigma}\norm{\vz - \vz'} \sigma_{n, j}(\vz') \\
         &\leq 2\sigma_{\max}L_{\sigma}\norm{\vz - \vz'}.
    \end{align*}
    \begin{align*}
            \left\vert \frac{1}{2}\sum^{d_x}_{j=1} \log \left(1 + \frac{\sigma_{n, j}^2(\vz)}{\sigma^{2}}\right) - \log \left(1 + \frac{\sigma_{n, j}^2(\vz')}{\sigma^{2}}\right) \right\vert
            &= \frac{1}{2}\left\vert \sum^{d_x}_{j=1} \log\left(1+\frac{
            \frac{\sigma^2_{n, j}(\vz) - \sigma^2_{n, j}(\vz') }{\sigma^2}
            }{1 + \frac{\sigma^2_{n, j}(\vz')}{\sigma^2}}\right)\right\vert \\
            &\leq \frac{1}{2} \left\vert \sum^{d_x}_{j=1} \log\left(1+
            \frac{\mid\sigma^2_{n, j}(\vz) - \sigma^2_{n, j}(\vz') \mid}{\sigma^2}\right)\right\vert \\
            &\leq \frac{1}{2 \sigma^2} \sum^{d_x}_{j=1} \left\vert\sigma^2_{n, j}(\vz) - 
            \sigma^2_{n, j}(\vz')\right\vert \tag{*}\\
            &\leq \frac{d_x\sigma_{\max} L_{\vsigma}}{\sigma^2} \norm{\vz - \vz'}.
\end{align*}
Where (*) is true because for all $x\geq 0$, $\log(1 + x) \leq x$.
\end{proof}

\begin{corollary}[\opacex gives an optimistic estimate on \Cref{eq:exploration_op}]
Let \Cref{assumption: Well Calibration Assumption} hold 
 and $\vpi_n^{*}$ denote the solution to \Cref{eq:exploration_op} and $J_n(\vpi_n^*)$ the resulting objective. Similarly, let $\vpi_n$ and $\veta_n$ be the solution to \Cref{eq:exploration_op_optimistic} and $J_n(\vpi_n, \veta_n)$ the corresponding value of the objective. Then with probability at least $1-\delta$ we have for every episode $n \in \set{1, \ldots, N}$:
\begin{equation*}
    J_n(\vpi_n^*) \le J_n(\vpi_n, \veta_n).
\end{equation*}
\label{cor:optimistic_estimate}
\end{corollary}
\begin{proof}
    Follows directly from \Cref{assumption: Well Calibration Assumption}.
\end{proof}

\subsection{Proof of \Cref{lem:optimistic_planning_regret}}
\begin{lemma}[Difference in Policy performance]
Let $J_{r, k}(\vpi, \vx_k) =  \E_{\vtau^{\vpi}}\left[\sum_{t=k}^{T-1} r(\vx_t, \vpi(\vx_{t}))\right]$ and $A_{r, k}(\vpi, \vx, \va) =  \E_{\vtau^{\vpi}}\left[r(\vx, \va) + J_{r, k + 1}(\vpi, \vx') - J_{r, k}(\vpi, \vx) \right]$ with $\vx' = \vf^*(\vx, \va) + \vw$. For simplicity we refer to  $J_{r, 0}(\vpi, \vx_0) = J_r (\vpi, \vx_0)$. 
The following holds for all $\vx_0 \in \setX$:
\begin{equation*}
    J_r (\vpi', \vx_0) -  J_r (\vpi, \vx_0) = \E_{\vtau^{\vpi'}} \left[\sum^{T-1}_{t=0} A_{r, t}(\vpi, \vx'_t, \vpi'(\vx'_t))\right]
\end{equation*}
\label{lemma:Difference in policy}
\end{lemma}
\begin{proof}
    \begin{align*}
         J_r (\vpi', \vx_0) &= \E_{\vtau^{\vpi'}} \left[\sum^{T-1}_{t=0} r(\vx'_t, \vpi'(\vx'_t))\right] = \E_{\vtau^{\vpi'}} \left[r(\vx_0, \vpi'(\vx_0)) + J_{r, 1}(\vpi', \vx'_1)\right] \\
         &= \E_{\vtau^{\vpi'}}  \left[r(\vx_0, \vpi'(\vx_0)) + J_{r, 1}(\vpi, \vx'_1) + J_{r, 1}(\vpi', \vx'_1) - J_{r, 1}(\vpi, \vx'_1)\right] \\
         &= \E_{\vtau^{\vpi'}}  \left[r(\vx_0, \vpi'(\vx_0)) + J_{r, 1}(\vpi, \vx'_1) - J_r(\vpi, \vx_0)\right] \\
         &+ J_r(\vpi, \vx_0) + \E_{\vtau^{\vpi'}}  \left[J_{r, 1}(\vpi', \vx'_1) - J_{r, 1}(\vpi, \vx'_1)\right] \\
        &= \E_{\vtau^{\vpi'}}  \left[A_{r, 0}(\vpi, \vx_0, \vpi'(\vx_0))\right] + J_r(\vpi, \vx_0) + \E_{\vtau^{\vpi'}}  \left[J_{r, 1}(\vpi', \vx_1) - J_{r, 1}(\vpi, \vx_1)\right]
    \end{align*}
    Therefore we obtain
    \begin{equation*}
        J_r (\vpi', \vx_0) - J_r(\vpi, \vx_0) = \E_{\vtau^{\vpi'}}  \left[A_0(\vpi, \vx_0, \vpi'(\vx_0))\right] +  \E_{\vtau^{\vpi'}}  \left[J_{r, 1}(\vpi', \vx'_1) - J_{r, 1}(\vpi, \vx'_1)\right].
    \end{equation*}
    Using the same argument for $J_{r, 1}$, $J_{r, 2}$, \dots, $J_{r, T-1}$ and that $J_{r, T}(\vpi, \vx) = 0$ for all $\vpi \in \Pi$ and $\vx \in \setX$ completes the proof.
\end{proof}
Assume a policy $\vpi$ is fixed and dynamics are of the form:
\begin{equation}
    \vx' = \vmu_n(\vx, \vpi(\vx)) + \beta_n(\delta) \vsigma(\vx, \vpi(\vx))\vu + \vw.
\end{equation}
Here $\vu \in [-1, 1]^{d_x}$. Furthermore, assume that the associated running rewards do not depend on $\vu$, that is, $r(\vx_t)$, and let $\veta \in \Xi$ denote the policy, i.e., $\veta: \setX \to [-1, 1]^{d_x}$.
\begin{corollary}
The following holds for all $\vx_0 \in \setX$ and policy $\vpi$:
\begin{equation*}
    J_r (\vpi, \veta', \vx_0) -  J_r (\vpi, \veta, \vx_0) = \E_{\vtau^{\veta'}}  \left[\sum^{T-1}_{t=0} J_{r, t+1}(\vpi, \veta, \vx'_{t+1}) -  J_{r, t+1}(\vpi, \veta, \vx_{t+1})\right],
\end{equation*}
with $\vx_{t+1} = \vmu_n(\vx'_t, \vpi(\vx'_t)) + \beta_n(\delta) \vsigma(\vx'_t, \vpi(\vx'_t))\veta(\vx'_t) + \vw_t$, and $\vx'_{t+1} = \vmu_n(\vx'_t, \vpi(\vx'_t)) + \beta_n(\delta) \vsigma(\vx'_t, \vpi(\vx'_t))\veta'(\vx'_t) + \vw_t$.
\label{cor:Difference in hallucinated policy}
\end{corollary}
\begin{proof}
    From \Cref{lemma:Difference in policy} we have
    \begin{equation*}
    J_r (\vpi, \veta', \vx_0) -  J_r (\vpi, \veta, \vx_0) = \E_{\vtau^{\veta'}}  \left[\sum^{T-1}_{t=0} A_{r, t}(\veta, \vx'_{t}, \veta'(\vx'_t))\right].
\end{equation*}
Furthermore, 
\begin{align*}
    \E_{\vtau^{\veta'}}\left[A_{r, t}(\veta, \vx'_{t}, \veta'(\vx'_t))\right] &= \E_{\vtau^{\veta'}} \left[r(\vx'_t) + J_{r, t+1}(\vpi, \veta, \vx'_{t+1}) - J_{r, t}(\vpi, \veta, \vx'_{t})\right] \\
    &= \E_{\vtau^{\veta'}} \left[r(\vx'_t) + J_{r, t+1}(\vpi, \veta, \vx'_{t+1}) - r(\vx'_t) - J_{r, t+1}(\vpi, \veta, \vx_{t+1})\right] \\
    &= \E_{\vtau^{\veta'}}\left[J_{r, t+1}(\vpi, \veta, \vx'_{t+1}) - J_{r, t+1}(\vpi, \veta, \vx_{t+1})\right].
\end{align*}
\end{proof}
\begin{proof}[Proof of \Cref{lem:optimistic_planning_regret}]
From \Cref{assumption: Well Calibration Assumption} we know that with probability at least $1-\delta$ there exists a $\Bar{\veta}$ such that $\vf^*(\vz) = \vmu_n(\vz) + \beta_n(\delta) \vsigma(\vz)\Bar{\veta}(\vx)$ for all $\vz \in \setZ$.
\begin{align*}
     J_n(\vpi_n^*) - J_n(\vpi_n) &\leq J_n(\vpi_n, \veta_n) - J_n(\vpi_n) \tag{\Cref{cor:optimistic_estimate}}\\
     &= J_n(\vpi_n, \veta_n) - J_n(\vpi_n, \Bar{\veta})\\
     &= \E_{\vtau^{\Bar{\veta}}} \left[\sum^{T-1}_{t=0} J_{n, t+1}(\vpi_n, \veta_n, \vx'_{t+1}) -  J_{n, t+1}(\vpi_n, \veta_n, \vx_{t+1})\right] \tag{\Cref{cor:Difference in hallucinated policy}}\\
    &= \E_{\vtau^{\vpi_n}}\left[\sum^{T-1}_{t=0} J_{n, t+1}(\vpi_n, \veta_n, \vx'_{t+1}) -  J_{n, t+1}(\vpi_n, \veta_n, \vx_{t+1})\right], \tag{Expectation wrt $\vpi_n$ under true dynamics $\vf^*$}\\
    &\text{with } \vx_{t+1} = \vf^*(\vx_{t}, \vpi_n(\vx_t)) + \vw_t, \\
    &\text{and } \vx'_{t+1} = \vmu_{n-1}(\vx_t, \vpi_n(\vx_t)) + \beta_{n-1}(\delta)\vsigma_{n-1}(\vx_t, \vpi_n(\vx_t)) \veta_n(\vx_t) + \vw_t.
\end{align*}
    
\end{proof}

\subsection{Analyzing regret of optimistic planning}
In the following, we analyze the regret of optimistic planning for both $\sigma$-Gaussian noise and $\sigma$-sub Gaussian noise case. We start with the Gaussian case.
\begin{lemma}[Absolute expectation Difference Under Two Gaussians (Lemma C.2.~\cite{kakade2020information})]
\label{lem:absolute_diff_gaussians}
For Gaussian distribution $\setN(\mu_1, \sigma^2 \mathbb{I})$ and $\setN(\mu_2, \sigma^2 \mathbb{I})$, and for any (appropriately measurable) positive function $g$, it holds that:
\begin{equation*}
    \left\lvert \E_{z \sim \setN_1}[g(z)] - \E_{z \sim \setN_2}[g(z)] \right\rvert \leq \min\left\{ \frac{\norm{\mu_1 - \mu_2}}{\sigma^2}, 1 \right\} \sqrt{\E_{z \sim \setN_1}[g^2(z)]}
\end{equation*}
\end{lemma}
\begin{proof}
    \begin{align*}
        \left\lvert \E_{z \sim \setN_1}[g(z)] - \E_{z \sim \setN_2}[g(z)] \right\rvert &= \left\lvert \E_{z \sim \setN_1} \left[ g(z) \left(1-\frac{\setN_{2}}{\setN_1}\right)\right]\right\rvert \\
        &\leq \left\lvert \sqrt{\E_{z \sim \setN_1} \left[ g^2(z)\right]} \sqrt{\E_{z \sim \setN_1}\left[\left(1-\frac{\setN_{2}}{\setN_1}\right)^2\right]}\right\rvert \\
        &= \sqrt{\E_{z \sim \setN_1} \left[ g^2(z)\right]} \sqrt{\E_{z \sim \setN_1}\left[\left(1-\frac{\setN_{2}}{\setN_1}\right)^2\right]} \\
        &\leq \sqrt{\E_{z \sim \setN_1}[g^2(z)]} \min\left\{ \frac{\norm{\mu_1 - \mu_2}}{\sigma^2}, 1 \right\} \tag{Lemma C.2.~\cite{kakade2020information}}
    \end{align*}
\end{proof}

\begin{corollary}[Regret of optimistic planning for Gaussian noise]
     Let $\vpi^*_n$, $\vpi_n$ denote the solution to \Cref{eq:exploration_op} and \Cref{eq:exploration_op_optimistic} 
 respectively, and $\vz^*_{n, t}$, $\vz_{n, t}$ the corresponding state-action pairs visited during their respective trajectories. Furthermore, let \Cref{assumption: Well Calibration Assumption} - \ref{ass:noise_properties} hold. Then, the following is true for all $n\geq 0$, $t \in  [0, T-1]$, with probability at least $1-\delta$
 \begin{align*}
     J_n(\vpi_n^*) - J_n(\vpi_n) &\leq \setO\left(T \E_{\vtau^{\vpi_n}} \left[\sum^{T-1}_{t=0} \frac{(1 + \sqrt{d_x}) \beta_{n-1}(\delta) \norm{\vsigma_{n-1}(\vz_{n, t})}}{\sigma^2}\right]\right)
 \end{align*}
\label{cor:bounding_diff_in_performance}
\end{corollary}
\begin{proof}
For simplicity, define $g_n(\vx) = J_{n, t+1}(\vpi_n, \veta_n, \vx)$.
Note since $\vw_t \sim \setN(0, \sigma^2 \mathbb{I})$ (\Cref{ass:noise_properties}), we have that $\vx'_{t+1}$ and $\vx_{t+1}$ are also Gaussians. Therefore, we can leverage \Cref{lem:absolute_diff_gaussians}.
\begin{align*}
    \E_{\vtau^{\vpi_n}} &\left[J_{n, t+1}(\vpi_n, \veta_n, \vx'_{t+1}) -  J_{n, t+1}(\vpi_n, \veta_n, \vx_{t+1})\right] = \E \left[g_n(\vx'_{t+1}) -  g_n(\vx_{t+1})\right] \\
    &\le  \sqrt{\E[g_n^2(\vx_{t+1})]} \min\left\{ \frac{\norm{\vx'_{t+1} - \vx_{t+1}}}{\sigma^2}, 1 \right\} \tag{\Cref{lem:absolute_diff_gaussians}} \\
   &\le \sqrt{J_{\max}} \min\left\{ \frac{\norm{\vx'_{t+1} - \vx_{t+1}}}{\sigma^2}, 1 \right\} \tag{\Cref{cor:bounded_objective}}.
\end{align*}
Furthermore, 
\begin{align*}
    \norm{\vx'_{t+1} - \vx_{t+1}} &= \norm{\vmu_{n-1}(\vx_t, \vpi_n(\vx_t)) + \beta_{n-1}(\delta)\vsigma_{n-1}(\vx_t, \vpi_n(\vx_t)) \veta_n(\vx_t) - \vf^*(\vx_t, \vpi_n(\vx_t))} \\
    &\le \norm{\vmu_{n-1}(\vx_t, \vpi_n(\vx_t)) - \vf^*(\vx_t, \vpi_n(\vx_t))} \\
    &+ \beta_{n-1}(\delta) \norm{\vsigma_{n-1}(\vx_t, \vpi_n(\vx_t))}\norm{\veta_n(\vx_t)} \\
    &\le (1 + \sqrt{d_x}) \beta_{n-1}(\delta) \norm{\vsigma_{n-1}(\vx_t, \vpi_n(\vx_t))} \tag{\Cref{assumption: Well Calibration Assumption}}.
\end{align*}
Next, we use \Cref{lem:optimistic_planning_regret} 
\begin{align*}
   J_n(\vpi_n^*) - J_n(\vpi_n)  &\leq  \E_{\vtau^{\vpi_n}} \left[\sum^{T-1}_{t=0} J_{n, t+1}(\vpi_n, \veta_n, \vx'_{t+1}) -  J_{n, t+1}(\vpi_n, \veta_n, \vx_{t+1})\right], \\
   &\leq \E_{\vtau^{\vpi_n}} \left[\sum^{T-1}_{t=0} \sqrt{J_{\max}} \min\left\{ \frac{(1 + \sqrt{d_x}) \beta_{n-1}(\delta) \norm{\vsigma_{n-1}(\vx_t, \vpi_n(\vx_t))}}{\sigma^2}, 1 \right\}\right], \\
   &\leq \sqrt{J_{\max}}  \E_{\vtau^{\vpi_n}} \left[\sum^{T-1}_{t=0} \frac{(1 + \sqrt{d_x}) \beta_{n-1}(\delta) \norm{\vsigma_{n-1}(\vx_t, \vpi_n(\vx_t))}}{\sigma^2}\right], \\
   &= \setO\left(T \E_{\vtau^{\vpi_n}} \left[\sum^{T-1}_{t=0} \frac{(1 + \sqrt{d_x}) \beta_{n-1}(\delta) \norm{\vsigma_{n-1}(\vx_t, \vpi_n(\vx_t))}}{\sigma^2}\right]\right).
\end{align*}

\end{proof}

\begin{lemma}[Regret of planning optimistically for sub-Gaussian noise]
\label{lem:simple_regret_planning}
    Let $\vpi^*_n$, $\vpi_n$ denote the solution to \Cref{eq:exploration_op} and \Cref{eq:exploration_op_optimistic} 
 respectively, and $\vz^*_{n, t}$, $\vz_{n, t}$ the corresponding state-action pairs visited during their respective trajectories. Furthermore, let \Cref{assumption: Well Calibration Assumption} and \ref{ass:lipschitz_continuity} hold, and relax \Cref{ass:noise_properties} to $\sigma$-sub Gaussian noise. Then, the following is true for all $n\geq 0$ with probability at least $1-\delta$
    \begin{equation*}
        J_n(\vpi_n^*) - J_n(\vpi_n) \leq \setO\left(L^{T-1}_{\vsigma} \beta^T_{n-1}(\delta) T \E_{\vtau^{\vpi_n}}\left[\sum^{T-1}_{t=0} \norm{ \vsigma_{n-1, j}(\vz_{n, t})} \right]\right)
    \end{equation*}
\end{lemma}
\begin{proof}
    \citet[Lemma~5]{curi2020efficient} bound the regret with the sum of epistemic uncertainties for Lipschitz continuous reward functions, under \Cref{assumption: Well Calibration Assumption} and~\ref{ass:lipschitz_continuity} for sub-Gaussian noise (c.f., \citet[Theorem 3.5]{rothfuss2023hallucinated} for a more rigorous derivation). For the active exploration setting, the reward in episode $n+1$ is 
    \begin{equation*}
        r(\vz) = \frac{1}{2}\sum^{d_x}_{j=1} \log \left(1 + \frac{\vsigma_{n-1, j}^2(\vz)}{\sigma^{2}}\right).
    \end{equation*}
    We show in \Cref{cor:bounded_objective} that our choice of exploration reward is Lipschitz continuous. Thus, can use the regret bound from \citet{curi2020efficient}.
\end{proof}

Compared to the Gaussian case, $\sigma$-sub Gaussian noise has the additional exponential dependence on the horizon $T$, i.e., the 
$\beta^T_n$ term. This follows from the analysis through Lipschitz continuity. Moreover, as we show in \Cref{lem:optimistic_planning_regret}, the regret of planning optimistically is proportional to the change in value under the same optimistic dynamics and policy, but different initial states. The Lipschitz continuity property of our objective allows us to relate the difference in values to the discrepancy in the trajectories. Even for linear systems, trajectories under the same dynamics and policy but different initial states can deviate exponentially in the horizon. 

\subsection{Proof for general Bayesian models}
In this section, we analyze the information gain for general Bayesian models and prove \Cref{thm:main_theorem}.
\begin{theorem}[Entropy of a RV with finite second moment is upper bounded by Gaussian entropy (Theorem 8.6.5~\cite{elementsofIT})]
    Let the random vector $\vx \in \R^n$ have covariance $\mK = \E\left[\vx\vx^{\top}\right]$ (i.e., $\mK_{ij} = E\left[\vx_i \vx_j\right] , 1 \leq i, j \leq n$). Then 
    \begin{equation*}
        H(X) \leq \frac{1}{2} \log((2\pi e)^n|\mK|)
    \end{equation*}
with equality if and only if $\vx \sim \setN(\vmu, \mK)$ for $\vmu = E\left[\vx\right]$.
\label{thm:entropy_bound}
\end{theorem}
\begin{lemma}[Monotonocity of information gain]
\label{lemma:info_never_hurts} Let $\vtau^{\vpi}$ denote the trajectory induced by the policy $\vpi$. Then, the following is true for all $n \geq 0$, policies $\vpi$
\begin{equation*}
   \E_{\setD_{1:n}}\left[I\left(\vf^*_{\vtau^{\vpi}}; \vy_{\vtau^{\vpi}} \mid \setD_{1:n}\right)\right] \leq   \E_{\setD_{1:n-1}}\left[I\left(\vf^*_{\vtau^{\vpi}}; \vy_{\vtau^{\vpi}} \mid \setD_{1:n-1}\right)\right]. 
\end{equation*}
\end{lemma}
\begin{proof}
\begin{align*}
    \E_{\setD_{1:n}}&\left[I\left(\vf^*_{\vtau^{\vpi}}; \vy_{\vtau^{\vpi}} \mid \setD_{1:n-1}\right)  - I\left(\vf^*_{\vtau^{\vpi}}; \vy_{\vtau^{\vpi}} \mid \setD_{1:n}\right)\right] \\
    &= 
    \E_{\setD_{1:n}}\left[H\left(\vy_{\vtau^{\vpi}} \mid \setD_{1:n-1}\right) - H\left(\vy_{\vtau^{\vpi}} \mid \vf^*_{\vtau^{\vpi}}, \setD_{1:n-1}\right)\right. \\
    &- \left. \left(
     H\left(\vy_{\vtau^{\vpi}} \mid \setD_{1:n}\right)
     - H\left(\vy_{\vtau^{\vpi}} \mid \vf^*_{\vtau^{\vpi}},  \setD_{1:n}\right) \right) \right] \\
    &= \E_{\setD_{1:n}}\left[H\left(\vy_{\vtau^{\vpi}} \mid \setD_{1:n-1}\right) - H\left(\vy_{\vtau^{\vpi}} \mid \setD_{1:n}\right)\right] \\
    &+ \E_{\setD_{1:n}}\left[H\left(\vy_{\vtau^{\vpi}} \mid \vf^*_{\vtau^{\vpi}}\right) - H\left(\vy_{\vtau^{\vpi}} \mid \vf^*_{\vtau^{\vpi}}\right)\right] \\
    &\geq 0 \tag{information never hurts}
\end{align*}

\end{proof}
A direct consequence of \Cref{lemma:info_never_hurts} is the following corollary.
\begin{corollary}[Information gain at $N$ is less than the average gain till $N$]
\label{cor:mutual_information_and_average_mi}
Let $\vtau^{\vpi}$ denote the trajectory induced by the policy $\vpi$. Then, the following is true for all $N \geq 1$, policies $\vpi$,
\begin{equation*}
     \E_{\setD_{1:N-1}}\left[I\left(\vf^*_{\vtau^{\vpi}}; \vy_{\vtau^{\vpi}} \mid \setD_{1:N-1}\right)\right]  \leq \frac{1}{N}\sum^{N}_{n=1} \E_{\setD_{1:n-1}}\left[I\left(\vf^*_{\vtau^{\vpi}}; \vy_{\vtau^{\vpi}} \mid \setD_{1:n-1}\right)\right].
\end{equation*}
\end{corollary}
Next, we prove \Cref{lem:uniform_exploration_upper}, which is central to our proposed active exploration objective in \Cref{eq:exploration_op}.
\begin{proof}[Proof of \Cref{lem:uniform_exploration_upper}]
Let $\vy_{\vtau^{\vpi}} = \{\vy_t\}^{T-1}_{t=0} = \{\vf^*_t + \vw_{t}\}^{T-1}_{t=0}$, where $\vf^*_t = \vf^*(\vz_{t})$. Furthermore, denote with $\Sigma_{n}(\vf^*_{0:T-1})$ the covariance of $\vf^*_{0:T-1}$.
    \begin{align*}
         I\left(\vf^*_{\vtau^{\vpi}}; \vy_{\vtau^{\vpi}} \mid \setD_{1:n}\right)
         &= I\left(\vf^*_{0:T-1}; \vy_{0:T-1} \mid \setD_{1:n}\right) \\
    &=H\left(\vy_{0:T-1} \mid \setD_{1:n}\right) - H\left(\vy_{0:T-1}\mid  \vf^*_{0:T-1}, \setD_{1:n}\right) \\
         &\leq \frac{1}{2}\log\left( \left\vert \sigma^2 \mathbb{I} + \Sigma_{n}(\vf^*_{0:T-1})
        \right\vert \right) - \frac{1}{2}\log\left( \left\vert \sigma^2 \mathbb{I}
        \right\vert \right) \tag{\Cref{thm:entropy_bound}}\\
         &\leq \frac{1}{2}\log\left( \left\vert \diag\left(\mathbb{I} + \sigma^{-2} \Sigma_{n}(\vf^*_{0:T-1})\right) \tag{Hadamard's inequality}
        \right\vert \right) \\
        &=  \frac{1}{2} \sum_{t=0}^{T-1} \sum^{d_x}_{j=1}\log \left(1 + \frac{\vsigma_{n, j}^2(\vz_t)}{\sigma^2}\right).
     \end{align*}
\end{proof}

We can leverage the result from \Cref{lem:uniform_exploration_upper}  to bound the average mutual information with the sum of epistemic uncertainties.
\begin{lemma}[Average information gain is less than sum of average epistemic uncertainties]\label{lem:average_mi_epistemic_uncertainties}
    Let \Cref{ass:noise_properties} hold and denote with $\bar{\vpi}_N$ be the solution of \Cref{eq:greedy_info_gain}.
    Then, for all $N \geq 1$ and dataset $\setD_{1:N}$ the following is true
\begin{align*}
     \frac{1}{N}&\sum^{N}_{n=1} \E_{\setD_{1:n-1}, \vtau^{\Bar{\vpi}_n}}\left[I\left(\vf^*_{\vtau^{\Bar{\vpi}_N}}; \vy_{\vtau^{\Bar{\vpi}_N}} \mid \setD_{1:n-1}\right)\right] \\
     &\leq \frac{1}{N}  \sum^{N}_{n=1} \E_{\setD_{1:n-1}, \vtau^{\vpi_n^*}}\left[\sum_{t=0}^{T-1} \sum_{j=1}^{d_x}\left(\frac{1}{2} \log \left(1 + \frac{\vsigma_{n, j}^2(\vz^*_{n, t})}{\sigma^2}\right)\right)\right],
\end{align*}
where $z^*_{n, t}$ are the state-action tuples visited by the solution of \Cref{eq:exploration_op}, i.e., $\vpi^*_n$.
\end{lemma}
\begin{proof}
    \begin{align*}
        \frac{1}{N}\sum^{N}_{n=1} &\E_{\setD_{1:n-1}, \vtau^{\Bar{\vpi}_N}}\left[I\left(\vf^*_{\vtau^{\Bar{\vpi}_N}}; \vy_{\vtau^{\Bar{\vpi}_N}} \mid \setD_{1:n-1}\right)\right] \\&\leq \frac{1}{N} \sum^{N}_{n=1} \E_{\setD_{1:n-1}, \vtau^{\Bar{\vpi}_N}}\left[  \left(\sum_{\vz_t \in \vtau^{\Bar{\vpi}_N}} \frac{1}{2} \sum_{j=1}^{d_x}  \log \left(1 + \frac{\vsigma_{n, j}^2(\vz_{t})}{\sigma^2}\right) \right)\right] \tag{\Cref{lem:uniform_exploration_upper} }\\
        &\leq \frac{1}{N}  \sum^{N}_{n=1} \E_{\setD_{1:n-1}}\left[ \max_{\vpi \in \Pi} \E_{ \vtau^{\vpi}}\left[ \sum_{\vz_t \in \vtau_{\vpi}} \sum_{j=1}^{d_x}  \left(\frac{1}{2} \log \left(1 + \frac{\vsigma_{n, j}^2(\vz_{t})}{\sigma^2}\right)\right)\right]\right] \tag{1}\\
        &= \frac{1}{N}  \sum^{N}_{n=1} \E_{\setD_{1:n-1}, \vtau^{\vpi_n}}\left[\sum_{t=0}^{T-1} \sum_{j=1}^{d_x} \left( \frac{1}{2}\log \left(1 + \frac{\vsigma_{n, j}^2(\vz^*_{n, t})}{\sigma^2}\right) \right)\right].
    \end{align*}
    Here (1) follows from the tower property. Note that the second expectation in (1) is wrt $\vtau^{\vpi}$ conditioned on a realization of $\setD_{1:n-1}$, where the conditioning is captured in the epistemic uncertainty $\vsigma_n(\cdot)$.
\end{proof}
We use the results from above, to prove \Cref{thm:main_theorem}.
\begin{proof}[Proof of  \cref{thm:main_theorem}]
Let $\bar{\vpi}_n$ denote the solution to \Cref{eq:greedy_info_gain} at iteration $n\geq 1$. We first relate the information gain from \opacex to the information gain of $\bar{\vpi}_n$.
    \begin{align*}
        \E_{\setD_{1:N-1}, \vtau^{\bar{\vpi}_N}}&\left[I\left(\vf^*_{\vtau^{\bar{\vpi}_N}}; \vy_{\vtau^{\bar{\vpi}_N}} \mid \setD_{1:N-1}\right)\right] \\
        &\leq \frac{1}{N}\sum^{N}_{n=1} \E_{\setD_{1:n-1}, \vtau^{\bar{\vpi}_n}}\left[I\left(\vf^*_{\vtau^{\bar{\vpi}_n}}; \vy_{\vtau^{\bar{\vpi}_n}} \mid \setD_{1:n-1}\right)\right] \tag{\Cref{cor:mutual_information_and_average_mi}}\\
        &\leq \frac{1}{N}  \sum^{N}_{n=1} \E_{\setD_{1:n-1}, \vtau^{\vpi_n}}\left[\left(\sum_{t=0}^{T-1} \sum^{d_x}_{j=1}\frac{1}{2} \log \left(1 + \frac{\vsigma_{n-1, j}^2(\vz^*_{n, t})}{\sigma^2}\right)\right)\right] \tag{\Cref{lem:average_mi_epistemic_uncertainties}} \\
        &= \frac{1}{N}  \sum^{N}_{n=1} \left(\E_{\setD_{1:n-1}, \vpi_n}\left[\sum_{t=0}^{T-1} \frac{1}{2} \sum^{d_x}_{j=1}\log \left(1 + \frac{\vsigma_{n-1, j}^2(\vz_{n, t})}{\sigma^2}\right)\right] + J_n(\vpi_n^*) - J_n(\vpi_n) \right) \\
        &\leq \frac{1}{N}  \sum^{N}_{n=1} \E_{\setD_{1:n-1}}\left[\E_{\vtau_{\vpi_n}}\left[\sum_{t=0}^{T-1} \frac{1}{2} \sum^{d_x}_{j=1}\log \left(1 + \frac{\vsigma_{n-1, j}^2(\vz_{n, t})}{\sigma^2}\right)\right] \right. \\
        &\left. \hspace{0.5cm} + \setO\left(\beta_{n-1}(\delta) T \E_{\vtau_{\vpi_n}}\left[\sum^{T-1}_{t=0}  \norm{\vsigma_{n-1}(\vz_{n, t})}_2 \right]\right)\right] \tag{\Cref{cor:bounding_diff_in_performance}}
        \end{align*}
                    In summary, the maximum expected mutual information at episode $N$ is less than the mutual information of \opacex and the sum of model epistemic uncertainties. Crucial to the proof is the regret bound for optimistic planning from \Cref{cor:bounding_diff_in_performance}. 
        \begin{align*}
        &\frac{1}{N}  \sum^{N}_{n=1} \E_{\setD_{1:n-1}}\left[\E_{\vtau_{\vpi_n}}\left[\sum_{t=0}^{T-1} \frac{1}{2} \sum^{d_x}_{j=1}\log \left(1 + \frac{\vsigma_{n-1, j}^2(\vz_{n, t})}{\sigma^2}\right)\right] \right.\\
        &\left. \hspace{0.5cm} + \setO\left(T\beta_{n-1}(\delta) \E_{\vtau_{\vpi_n}}\left[\sum^{T-1}_{t=0}  \norm{\vsigma_{n-1}(\vz_{n, t})}_2 \right]\right)\right] \\
        &= \frac{1}{N}  \sum^{N}_{n=1} \E_{\setD_{1:n-1}}\left[\E_{\vtau_{\vpi_n}}\left[\sum_{t=0}^{T-1}  \sum^{d_x}_{j=1}\log \left(\sqrt{1 + \frac{\vsigma_{n-1, j}^2(\vz_{n, t})}{\sigma^2}}\right)\right] \right. \\
        &\left. \hspace{0.5cm} + \setO\left(T\beta_{n-1}(\delta) \E_{\vtau_{\vpi_n}}\left[\sum^{T-1}_{t=0}  \norm{\vsigma_{n-1}(\vz_{n, t})}_2 \right]\right)\right] \\
        &\leq \frac{1}{N}  \sum^{N}_{n=1} \E_{\setD_{1:n-1}}\left[\E_{\vtau_{\vpi_n}}\left[\sum_{t=0}^{T-1}  \sum^{d_x}_{j=1}\log \left(1 + \frac{\vsigma_{n-1, j}(\vz_{n, t})}{\sigma}\right)\right] \right. \\
        &\left. \hspace{0.5cm} + \setO\left(T\beta_{n-1}(\delta) \E_{\vtau_{\vpi_n}}\left[\sum^{T-1}_{t=0}  \norm{\vsigma_{n-1}(\vz_{n, t})}_2 \right]\right)\right] \\
        &\leq \frac{1}{N}  \sum^{N}_{n=1} \E_{\setD_{1:n-1}}\left[\E_{\vtau_{\vpi_n}}\left[\sum_{t=0}^{T-1} \sum^{d_x}_{j=1}\frac{\vsigma_{n-1, j}(\vz_{n, t})}{\sigma}\right] \right. \\
        &\left. \hspace{0.5cm} + \setO\left(T\beta_{n-1}(\delta) \E_{\vtau_{\vpi_n}}\left[\sum^{T-1}_{t=0}  \norm{\vsigma_{n-1}(\vz_{n, t})}_2 \right]\right)\right] \tag{$\log(1 + x) \leq x$ for $x \geq 0$.}\\
        &\leq \setO\left( \frac{1}{N}  \sum^{N}_{n=1} \E_{\setD_{1:n-1}}\left[T \beta_{n-1}(\delta)\E_{\vtau_{\vpi_n}}\left[\sum^{T-1}_{t=0}  \norm{\vsigma_{n-1}(\vz_{n, t})}_2 \right]\right]\right)
        \end{align*}
        Above, we show that the maximum expected mutual information can be upper bounded with the sum of epistemic uncertainties for the states \opacex visits during learning. Finally, we further upper bound this with the model complexity measure. 
        \begin{align*}
        &\setO\left( \frac{1}{N}  \sum^{N}_{n=1} \E_{\setD_{1:n-1}}\left[\beta_{n-1}(\delta) T \E_{\vtau_{\vpi_n}}\left[\sum^{T-1}_{t=0}  \norm{\vsigma_{n-1}(\vz_{n, t})}_2 \right]\right]\right) \\
        &= \setO\left( \frac{1}{N}\sqrt{\left(\E_{\setD_{1:N}}\left[\sum^{N}_{n=1} (T \beta_{n-1} (\delta))\E_{\vtau_{\vpi_n}}\left[\sum^{T-1}_{t=0} \norm{\vsigma_{n-1}(\vz_{n, t})}_2\right]\right]\right)^2}\right) \\
        &\leq \setO\left(\frac{1}{N} T \beta_{N}(\delta))
 \sqrt{T N \E_{\setD_{1:N}}\left[\sum^{N}_{n=1}\E_{\vtau_{\vpi_n}}\left[\sum^{T-1}_{t=0} ||\vsigma_{n}(\vz_{n, t})||^2_2\right]\right]}\right) \\
        &\leq \setO\left(\beta_{N}(\delta) T^{\sfrac{3}{2}}
 \sqrt{\frac{{\setM\setC}_N(\vf^*)}{N}}\right)
    \end{align*}
\end{proof}
\Cref{thm:main_theorem} gives a bound on the maximum expected mutual information w.r.t.~the model complexity. We can use concentration inequalities such as Markov, to give a high probability bound on the information gain. In particular, we have for all $\epsilon > 0$
\begin{align*}
     \Pr \left(I\left(\vf^*_{\vtau^{\bar{\vpi}_N}}; \vy_{\vtau^{\bar{\vpi}_N}} \mid \setD_{1:N-1}\right)  \geq \epsilon \right) &\leq  \frac{ \E_{\setD_{1:N-1}, \vtau^{\bar{\vpi}_N}}\left[I\left(\vf^*_{\vtau^{\bar{\vpi}_N}}; \vy_{\vtau^{\bar{\vpi}_N}} \mid \setD_{1:N-1}\right)\right]}{\epsilon} \\
    \\ &\leq  \setO \left(T^{\sfrac{3}{2}}\beta_{N}(\delta) \sqrt{\frac{{\setM\setC}_N(\vf^*)}{N\epsilon^2}}\right).
\end{align*}

\subsection{Proof of GP results}
This section presents our results for the frequentist setting where the dynamics are modeled using GPs. Since the information gain has no meaning in the frequentist setting, we study the epistemic uncertainty of the GP models.
\begin{corollary}[Monotonicity of the variance]
For all $n \geq 0$, and policies $\vpi$ the following is true.
\begin{equation*}
    \sum_{t=0}^{T-1} \sum_{j=1}^{d_x}  \frac{1}{2} \log \left(1 + \frac{\sigma_{N-1, j}^2(\vz_{t})}{\sigma^2}\right)  \leq \frac{1}{N}\sum^{N}_{n=1} \sum_{t=0}^{T-1} \sum_{j=1}^{d_x}  \frac{1}{2} \log \left(1 + \frac{\vsigma_{n-1, j}^2(\vz_{t})}{\sigma^2}\right)
\end{equation*}
\label{cor:monotonocity_of_variance}
\end{corollary}
\begin{proof}
    Follows directly due to the monotonicity of GP posterior variance.
\end{proof}
Next, we prove that the trajectory of \Cref{eq:exploration_op} at iteration $n$ is upper-bounded with the maximum information gain.
\begin{lemma}
    Let \Cref{ass:lipschitz_continuity} - \ref{ass:rkhs_func} hold Then, for all $N \geq 1$, with probability at least $1-\delta$, we have
    \begin{equation*}
    \max_{\vpi \in \Pi} \E_{\vtau_{\vpi}}\left[\left(\sum_{t=0}^{T-1} \sum^{d_x}_{j=1}\frac{1}{2} \log \left(1 + \frac{\sigma_{N, j}^2(\vz_{t})}{\sigma^2}\right)\right)\right] \leq \setO \left(\beta_{N}(\delta) T^{\sfrac{3}{2}}\sqrt{\frac{\gamma_N}{N}}\right).
\end{equation*}
Moreover, relax noise \Cref{ass:noise_properties} to $\sigma$-sub Gaussian. 
Then, for all $N \geq 1$, with probability at least $1-\delta$, we have
\begin{equation*}
    \max_{\vpi \in \Pi} \E_{\vtau_{\vpi}}\left[\left(\sum_{t=0}^{T-1} \sum^{d_x}_{j=1}\frac{1}{2} \log \left(1 + \frac{\sigma_{N, j}^2(\vz_{t})}{\sigma^2}\right)\right)\right] \leq \setO \left(L^{T}_{\vsigma} \beta^{T}_{N}(\delta) T^{\sfrac{3}{2}}\sqrt{\frac{\gamma_N}{N}}\right)
\end{equation*}
\label{lem:Gp_bound_general}
\end{lemma}
\begin{proof}
\textbf{Gaussian noise case}:
Let $\vz^*_{n, t}$ denote the state-action pair at time $t$ for the trajectory of \Cref{eq:exploration_op} at iteration $n\geq 1$ and $\vpi^*_n$ the corresponding policy.
   \begin{align*}
    &\E_{\vtau_{\vpi^*_n}}\left[\left(\sum_{t=0}^{T-1} \sum^{d_x}_{j=1}\frac{1}{2} \log \left(1 + \frac{\sigma_{N, j}^2(\vz^*_{N, t})}{\sigma^2}\right)\right)\right] \\
          &\leq \frac{1}{N} \sum^N_{n=1} \E_{\vtau_{\vpi^*_n}}\left[\sum_{t=0}^{T-1} \sum^{d_x}_{j=1}\frac{1}{2} \log \left(1 + \frac{\sigma_{n, j}^2(\vz^*_{N, t})}{\sigma^2}\right)\right] \tag{\Cref{cor:monotonocity_of_variance}}\\
          &\leq \frac{1}{N} \sum^N_{n=1} \E_{\vtau_{\vpi^*_n}}\left[\sum_{t=0}^{T-1} \sum^{d_x}_{j=1}\frac{1}{2} \log \left(1 + \frac{\sigma_{n, j}^2(\vz^*_{n, t})}{\sigma^2}\right)\right] \tag{By definition of $\vpi^*_n$}\\
          &\leq  \setO \left(\beta_{N}(\delta) T^{\sfrac{3}{2}}\sqrt{\frac{{\setM\setC}_N(\vf^*)}{N}}\right) \tag{See proof of \Cref{thm:main_theorem}}\\
          &\leq \setO \left(\beta_{N}(\delta) T^{\sfrac{3}{2}}\sqrt{\frac{\gamma_N}{N}}\right) \tag*{\cite[Lemma 17]{curi2020efficient}}
   \end{align*}
   \textbf{Sub-Gaussian noise case}: The only difference between the Gaussian and sub-Gaussian case is the regret term (c.f., \Cref{cor:bounding_diff_in_performance} and \Cref{lem:simple_regret_planning}). In particular, the regret for the sub-Gaussian case leverages the Lipschitz continuity properties of the system (\Cref{ass:lipschitz_continuity}). This results in an exponential dependence on the horizon for our bound. We refer the reader to \citet{curi2020efficient, rothfuss2023hallucinated} for a more detailed derivation.
   \end{proof}
   \Cref{lem:Gp_bound_general} gives a sample complexity bound that holds for a richer class of kernels. Moreover, for GP models, $\beta_N \propto \sqrt{\gamma_N}$~\citep{chowdhury2017kernelized}. Therefore, for kernels, where $\lim_{N \to \infty} \sfrac{\gamma^{2}_N}{N} \to 0$, we can show convergence (for the Gaussian case). 
We summarize bounds on $\gamma_N$ from \cite{vakili2021information} in \Cref{table: gamma magnitude bounds for different kernels}.
   \begin{table}[ht!]
\begin{center}
\caption{Maximum information gain bounds for common choice of kernels.}
\label{table: gamma magnitude bounds for different kernels}
\begin{tabular}{@{}lll@{}}
\toprule
Kernel&$k(\vx, \vx')$ & $\gamma_n$ \\ \midrule
    Linear &$\vx^\top \vx'$   & $\mathcal{O}\left(d \log(n)\right)$                    \\
    RBF &$e^{-\frac{\norm{\vx - \vx'}^2}{2l^2}}$& $\mathcal{O}\left( \log^{d+1}(n)\right)$                    \\
    Matèrn &$\frac{1}{\Gamma(\nu)2^{\nu - 1}}\left(\frac{\sqrt{2\nu}\norm{\vx-\vx'}}{l}\right)^{\nu}B_{\nu}\left(\frac{\sqrt{2\nu}\norm{\vx-\vx'}}{l}\right)$  & $\mathcal{O}\left(n^{\frac{d}{2\nu + d}}\log^{\frac{2\nu}{2\nu+d}}(n)\right)$                     \\ \bottomrule
\end{tabular}
\end{center}
\end{table}

   From hereon, we focus on deriving the results for the case of Gaussian noise case. All our results can be easily extended for the sub-Gaussian setting by considering its corresponding bound.
   \begin{lemma}
       The following is true for all $N \geq 0$ and policies $\vpi \in \Pi$,
       \begin{equation*}
           \E_{\vtau_{\vpi}}\left[ \max_{\vz \in \vtau^{\vpi}}\sum^{d_x}_{j=1} \frac{1}{2}\vsigma_{N, j}^2(\vz) \right] \leq C_{\vsigma} \E_{\vtau_{\vpi}}\left[\sum_{t=0}^{T-1} \sum^{d_x}_{j=1}\frac{1}{2} \log \left(1 + \frac{\vsigma_{N, j}^2(\vz_{ t})}{\sigma^2}\right)\right],
       \end{equation*}
       with $C_{\vsigma} = \frac{\vsigma_{\max}}{\log(1 + \sigma^{-2}\vsigma_{\max})}$.
   \end{lemma}
\begin{proof}
\begin{align*}
&C_{\vsigma} \E_{\vtau_{\vpi}}\left[\sum_{t=0}^{T-1} \sum^{d_x}_{j=1}\frac{1}{2} \log \left(1 + \frac{\vsigma_{N, j}^2(\vz_{ t})}{\sigma^2}\right)\right] \\
&\geq  \E_{\vtau_{\vpi}}\left[ \sum_{t=0}^{T-1} \sum^{d_x}_{j=1} \frac{1}{2}\vsigma_{N, j}^2(\vz_{t}) \right] \tag*{\cite[Lemma.~15]{curi2020efficient}}, \\ 
&\geq \E_{\vtau_{\vpi}}\left[ \max_{\vz \in \vtau^{\vpi}}\sum^{d_x}_{j=1} \frac{1}{2}\vsigma_{N, j}^2(\vz) \right].
\end{align*}
\end{proof}
\begin{corollary}
    Let \Cref{ass:lipschitz_continuity} and \ref{ass:rkhs_func} hold, and relax noise \Cref{ass:noise_properties} to $\sigma$-sub Gaussian. 
Then, for all $N \geq 1$, with probability at least $1-\delta$, we have
\begin{equation*}
    \max_{\vpi \in \Pi} \E_{\vtau_{\vpi}}\left[ \max_{\vz \in \vtau^{\vpi}}\sum^{d_x}_{j=1} \frac{1}{2}\vsigma_{N, j}^2(\vz) \right] \leq \setO \left(\beta_{N}(\delta) T^{\sfrac{3}{2}}\sqrt{\frac{\gamma_N}{N}}\right).
\end{equation*}
\label{cor:max_uncertainty_bound}
\end{corollary}
\begin{lemma}
     Let \Cref{ass:lipschitz_continuity} and \ref{ass:rkhs_func} hold, and relax noise \Cref{ass:noise_properties} to $\sigma$-sub Gaussian. Furthermore, assume $\lim_{N\to \infty} \sfrac{\beta_N^2(\delta)\gamma_N(k)}{N} \to 0$. Then for all $N \geq 1$, $\vz \in \setR$, and $1 \leq j \leq d_x$, with probability at least $1-\delta$, we have
     \begin{equation*}
         \vsigma_{n, j}(\vz) \xrightarrow[]{\mathrm{a.s.}} 0 \text{ for } n \to \infty.
     \end{equation*}
     \label{lem:as_convergence_lemma}
\end{lemma}
\begin{proof}
We first show that the expected epistemic uncertainty along a trajectory converges to zero almost surely. Then we leverage this result to show almost sure convergence of all trajectories induced by $\vpi \in \Pi$. To this end, let $S_n = \E_{\vtau_{\vpi^*_n}}\left[\left(\sum_{t=0}^{T-1} \sum^{d_x}_{j=1}\frac{1}{2} \log \left(1 + \frac{\sigma_{n, j}^2(\vz^*_{n, t})}{\sigma^2}\right)\right)\right]$ for all $n\geq 0$. So far we have,
   \begin{equation*}
       \Pr \left( S_n \leq \setO \left(\beta_{N}(\delta) T^{\sfrac{3}{2}}\sqrt{\frac{\gamma_n}{n}}\right) \right) \geq 1 - \delta
   \end{equation*}
Consider a sequence $\{\delta_n\}_{n\geq 0}$ such that $\lim_{n \to \infty} \delta_n = 0$, and $\lim_{n \to \infty} \beta_{n}(\delta_n) T^{\sfrac{3}{2}}\sqrt{\frac{\gamma_n}{n}} \to 0$. Note, for GP models with $\lim_{N\to \infty} \sfrac{\beta_N^2(\delta)\gamma_N(k)}{N} \to 0$, such a sequence of $\delta_n$ exists~\citep{chowdhury2017kernelized}.
Consider any $\epsilon > 0$ and
let $N^*(\epsilon)$ be the smallest integer such that 
\begin{equation*}
    \setO \left(\beta_{N^*(\epsilon)}(\delta) T^{\sfrac{3}{2}}\sqrt{\frac{\gamma_{N^*(\epsilon)}}{N^*(\epsilon)}}\right) < \epsilon.
\end{equation*}
Then, we have
\begin{align*}
    \sum^{\infty}_{n=0}\Pr(S_n > \epsilon) &=
    \sum^{N^*(\epsilon)-1}_{n=0} \Pr(S_n > \epsilon) + \sum^{\infty}_{n=N^*(\epsilon)} \Pr(S_n > \epsilon) \\
    &= \sum^{N^*(\epsilon)-1}_{n=0} \Pr(S_n > \epsilon) + \sum^{\infty}_{n=N^*(\epsilon)} \delta_n \\
    &\leq N^*(\epsilon) + \sum^{\infty}_{n=N^*(\epsilon)} \delta_n.
\end{align*}
Note, since $\lim_{n \to \infty} \delta_n = 0$, we have 
\begin{equation*}
    \sum^{\infty}_{n=N^*(\epsilon)} \delta_n < \infty.
\end{equation*}
In particular, $\sum^{\infty}_{n=0}\Pr(S_n > \epsilon) < \infty$ for all $\epsilon > 0$.
 Therefore, we obtain 
\begin{equation*}
  S_n \xrightarrow[]{\mathrm{a.s.}} 0 \text{ for } n \to \infty. 
\end{equation*}
Define the random variable $V = \lim_{n\to \infty} \left(\sum_{t=0}^{T-1} \sum^{d_x}_{j=1}\frac{1}{2} \log \left(1 + \frac{\sigma_{n, j}^2(\vz^*_{n, t})}{\sigma^2}\right)\right)$, with $\vz^*_{n, t} \in \vtau$ and $ \vtau \sim \vtau^{\vpi^*_n}$. $V$ represents the sum of epistemic uncertainties of a random trajectory induced by the sequence of policies $\{\vpi_{n}\}_{n \geq 0}$. Note $V \geq 0$, therefore we apply Markov's inequality. Moreover, for all $\epsilon > 0$, we have
   \begin{equation*}
          \Pr \left(V > \epsilon \right) 
          \leq \frac{\E[V]}{\epsilon} = 0.
   \end{equation*}
   Hence, we have 
   \begin{equation*}
             \Pr \left(V = 0 \right) = 1 \implies V\xrightarrow[]{\mathrm{a.s.}} 0 \text{ for } n \to \infty
\end{equation*}
   Accordingly, we get for all $\pi \in \Pi$. 
   \begin{equation}
       \Pr\left(\lim_{n \to \infty} \sum_{\vz_{t} \in \vtau_{\vpi}} \sum^{d_x}_{j=1}\frac{1}{2} \log \left(1 + \frac{\sigma_{n, j}^2(\vz_{t})}{\sigma^2}\right) \to 0 \right) = 1.
       \label{eq:traj_convergence}
   \end{equation}
   \looseness=-1
   Assume there exists a $\vz \in  \setR$, such that for some $\epsilon$, $\sigma_{n, j}^2(\vz) > \epsilon$ for all $n \geq 0$. Since, $\vz \in \setR$, there exists a $t$ and $\pi \in \Pi$ such that $p(\vz_t = \vz|\pi, \vf^*) > 0$. This implies that $\Pr(\vz \in \vtau_{\vpi}) > 0$. However, from \Cref{eq:traj_convergence}, we have that $\sigma_{n, j}^2(\vz) \to 0$ for $N \to \infty$ almost surely, which is a contradiction.
\end{proof}
Finally, we leverage the results from above to prove \Cref{thm:gp_theorem}.
\begin{proof}[Proof of \Cref{thm:gp_theorem}]
    The proof follows directly from \Cref{cor:max_uncertainty_bound} and \Cref{lem:as_convergence_lemma}.
\end{proof}

\subsection{Zero-shot guarantees} \label{sec:zero_shot_performance_theory}
In this section, we give guarantees on the zero-shot performance of \opacex for a bounded cost function. We focus this section on the case of Gaussian noise. However, a similar analysis can be performed for the sub-Gaussian case and Lipschitz continuous costs. Since the analysis for both cases is similar, we only present the Gaussian case with bounded costs here.
\begin{corollary}\label{cor:pessimistic policy difference}
    Consider the following optimal control problem
\begin{align}
        \label{eq:oc_cost}
    \underset{\vpi \in \Pi}{\arg\min} \hspace{0.2em} &J_c(\vpi, \vf^*) = \underset{\vpi \in \Pi} {\arg\min} \hspace{0.2em}\E_{\vtau^{\vpi}}\left[\sum^{T-1}_{t=0} c(\vx_t, \vpi(\vx_t)) \right], \\
        \vx_{t+1} &= \vf^*(\vx_t, \vpi(\vx_t)) + \vw_t \quad \forall 0\leq t\leq T, \notag 
\end{align}
with bounded and positive costs.
Then we have for all policies $\vpi$ with probability at least $1-\delta$
\begin{equation*}
    J_c(\vpi, \veta^{P}) - J_c(\vpi) \leq  \setO\left(T \E_{\vtau^{\vpi_n}} \left[\sum^{T-1}_{t=0} \frac{(1 + \sqrt{d_x}) \beta_{n-1}(\delta) \norm{\vsigma_{n-1}(\vz_{ t})}}{\sigma^2}\right]\right),
\end{equation*}
where $J_c(\vpi, \veta^{P})  = \max_{\veta \in \Xi} J_c(\vpi, \veta)$. 
\end{corollary}

\begin{proof}
From \Cref{cor:Difference in hallucinated policy} we get
\begin{equation*}
    J_c (\vpi, \veta^P) - J_c (\vpi) = \E_{\vtau^{\vpi}}  \left[\sum^{T-1}_{t=0} J_{r, t+1}(\vpi, \veta^P, \vx^P_{t+1}) -  J_{r, t+1}(\vpi, \veta^P, \vx_{t+1})\right],
\end{equation*}
with $\vx_{t+1} =  \vf^*(\vx_t, \vpi(\vx_t)) + \vw_t$, and $\vx^P_{t+1} = \vmu_n(\vx_t, \vpi(\vx_t)) + \beta_n(\delta) \vsigma(\vx_t, \vpi(\vx_t))\veta^P(\vx_t) + \vw_t$.
Furthermore, the cost is positive and bounded. Therefore,
\begin{equation*}
    J^2_c(\vpi, \veta, \vx) \leq T^2 c_{\max}^2,
\end{equation*}
for all $\vx, \veta, \vpi$.
Accordingly, we can now use the same analysis as in \Cref{lem:optimistic_planning_regret} and get  
\begin{equation*}
    J_c(\vpi, \veta^{P}) - J_c(\vpi)  \leq  \setO\left(T \E_{\vtau^{\vpi_n}} \left[\sum^{T-1}_{t=0} \frac{(1 + \sqrt{d_x}) \beta_{n-1}(\delta) \norm{\vsigma_{n-1}(\vz_{ t})}}{\sigma^2}\right]\right),
\end{equation*}
\end{proof}

\begin{lemma}
    Consider the control problem in \Cref{eq:oc_cost} and
let \Cref{ass:noise_properties} and \ref{ass:rkhs_func} hold. 
Furthermore, assume for every $\epsilon > 0$, there exists a finite integer $n^*$ such that
\begin{equation}
     \forall n \geq n^*;  
    \beta^{\sfrac{3}{2}}_n(\delta) T^{\sfrac{11}{4}}\left(\frac{{\gamma}_{n}(k)}{n}\right)^{\frac{1}{4}} \le \epsilon, \label{eq:decreanse_condition}
 \end{equation}
 and denote with $\hat{\vpi}_n$ the minimax optimal policy, i.e., the solution to $\min_{\vpi \in \Pi}\max_{\veta \in \Xi} J_c(\vpi, \veta)$. Then for all $n \geq n^*$, we have probability at least $1-\delta$, $J_c(\hat{\vpi}_N) - J_c(\vpi^*) \le \setO(\epsilon)$.
 \label{lem:zero_shot_performance}
\end{lemma}
\begin{proof}[Proof of Zero-shot performance]
In \Cref{thm:gp_theorem} we give a rate at which the maximum uncertainty along a trajectory decreases:
\begin{equation*}
     \max_{\vpi \in \Pi} \E_{\vtau^{\vpi}}\left[ \max_{ \vz \in \vtau^{\vpi}} \frac{1}{2}\norm{\sigma_{N, j}(\vz)}^2 \right] \le \setO \left(\beta_N T^{\sfrac{3}{2}}\sqrt{\frac{{\gamma}_{N}(k)}{N}}\right).
\end{equation*}
Combining this with \Cref{cor:pessimistic policy difference} we get
\begin{align*}
    J_c(\vpi, \veta^{P}) - J_c(\vpi)  &\leq  \setO\left(T \E_{\vtau^{\vpi_n}} \left[\sum^{T-1}_{t=0} \frac{(1 + \sqrt{d_x}) \beta_{n-1}(\delta) \norm{\vsigma_{n-1}(\vz_{ t})}}{\sigma^2}\right]\right) \\
    &\le \setO\left( T^2 \beta_{n-1}(\delta)\E_{\vtau^{\vpi_n}} \left[ \max_{ \vz \in \vtau^{\vpi_n}}  \norm{\vsigma_{n-1}(\vz)}\right]\right) \\
    &\le \setO\left(\sqrt{\E_{\vtau^{\vpi_n}} \left[ \max_{ \vz \in \vtau^{\vpi_n}}  T^4 \beta^2_{n-1}(\delta)\norm{\vsigma_{n-1}(\vz)}^2\right]}\right) \\
    &\le \setO\left(\sqrt{
    \beta^3_n T^{\sfrac{11}{2}}\sqrt{\frac{{\gamma}_{n}(k)}{n}}}\right) \\
    &= \setO\left(
    \beta^{\sfrac{3}{2}}_n(\delta) T^{\sfrac{11}{4}}\left(\frac{{\gamma}_{n}(k)}{n}\right)^{\frac{1}{4}}\right) \\
    &= \setO(\epsilon) \tag{$\forall n \geq n^*$}.
\end{align*}
Hence, we have that for each policy $\vpi$, our upper bound $\max_{\veta} J_c(\vpi, \veta^{P})$ is $\epsilon$ precise, i.e.,
\begin{equation}
\label{eq:precision result}
    \max_{\veta \in \Xi} J_c(\vpi, \veta) - J_c(\vpi)  \leq \setO(\epsilon), \forall \vpi \in \Pi.
\end{equation}
We leverage this to prove optimality for the minimax solution. For the sake of contradiction, assume that 
\begin{equation}
\label{eq:contradictory_ass}
   J_c(\Hat{\vpi}_n) > J_c(\vpi^*) + \setO(\epsilon). 
\end{equation}
Then we have,
\begin{align*}
    \max_{\veta \in \Xi} J_c(\vpi^*, \veta) &\ge \min_{\vpi \in \Pi }\max_{\veta \in \Xi} J_c(\vpi, \veta) \\
    &= J_c(\Hat{\vpi}_n, \Hat{\veta}^P) \\
    &= \max_{\veta \in \Xi} J_c(\Hat{\vpi}_n, \veta) \\
    &\ge J_c(\Hat{\vpi}_n) \\
    &> J_c(\vpi^*) + \setO(\epsilon) \tag{\Cref{eq:contradictory_ass}}\\
    &\geq J_c(\vpi^*, \veta^{*, P}) \tag{\Cref{eq:precision result}}\\
    &= \max_{\veta \in \Xi} J_c(\vpi^*, \veta).
\end{align*}
This is a contradiction, which completes the proof.
\end{proof}
\Cref{lem:zero_shot_performance} shows that \opacex also results in nearly-optimal zero-shot performance. The convergence criteria in \Cref{eq:decreanse_condition} is satisfied for kernels $k$ that induce a very rich class of RKHS (c.f., \Cref{table: gamma magnitude bounds for different kernels}).

\section{Experiment Details} 
\label{sec:experiment_details}
The environment details and hyperparameters used for our experiments are presented in this section. In \Cref{subsec:fetch_pick} we discuss the experimental setup of the Fetch Pick \& Construction environment~\citep{li2020towards} in more detail. 
\subsection{Environment Details}
\Cref{tab:environment_rewards} lists the rewards used for the different environments. 
\begin{table}[htb]
    \centering
    \caption{Downstream task rewards for the environments presented in \Cref{sec:experiments}.}
    \label{tab:environment_rewards}
\begin{tabular}{cc}
\midrule
\textbf{Environment}                   & \textbf{Reward} $r_t$                                                                                                      \\
\midrule
Pendulum-v1 - swing up        & $\theta_t^2 + 0.1 \Dot{\theta}_t + 0.001 u_t^2$                      \\
Pendulum-v1 - keep down       & $(\theta_t-\pi)^2 + 0.1 \Dot{\theta}_t + 0.001 u_t^2$  \\
MountainCar                   & $-0.1 u_t^2  + 100 (1\{\vx_t \in \vx_{\text{goal}}\})$                            \\
Reacher - go to target        & 
$100 (1\{||\vx_t - \vx_{\text{target}}||_2\leq 0.05\}$ \\
Swimmer - go to target        & $-||\vx_t - \vx_{\text{target}}||_2$                                                                                     \\
Swimmer - go away from target & $||\vx_t - \vx_{\text{target}}||_2$                                                                                      \\
Cheetah - run forward         & $v_{\text{forward}, t}$                                                                                                 \\
Cheetah - run backward        & $-v_{\text{forward}, t}$              \\
\hline \\
\end{tabular}
\end{table}
We train the dynamics model after each episode of data collection. For training, we fix the number of epochs to determine the number of gradient steps. Furthermore, for computational reasons, we upper bound the number of gradient steps by ``maximum number of gradient steps''. The hyperparameters for the model-based agent are presented in \Cref{tab:environment_hyperparams}. Furthermore, we present the iCEM hyperparameters in \Cref{tab:icem_params}. For more detail on the iCEM hyperparameters see~\cite{iCem}.
\begin{table}[ht]
\centering
    \caption{Hyperparameters for results in~\Cref{sec:experiments}.}
    \label{tab:environment_hyperparams}
\begin{adjustbox}{max width=\linewidth}\begin{threeparttable}
\begin{tabular}{ccccccc}
\midrule
\textbf{Hyperparameters}         & \textbf{Pendulum-v1 - GP} & \textbf{Pendulum-v1} & \textbf{MountainCar} & \textbf{Reacher} & \textbf{Swimmer} & \textbf{Cheetah} \\
\midrule
Action repeat           & N/A              & N/A         & N/A         & N/A     & 4       & 4       \\
Exploration horizon     & $100$              & $200$         & $200$         & $50$      & $1000$    & $1000$    \\
Downstream task horizon & $200$              & $200$         & $200$         & $50$      & $1000$    & $1000$    \\
Hidden layers           & N/A              & $2$           & $2$           & $2$       & $4$       & $4$       \\
Neurons per layers      & N/A              & $256$         & $128$         & $256$     & $256$     & $256$     \\
Number of ensembles     & N/A              & $7$           & $5$           & $5$       & $5$       & $5$       \\
Batch size              & N/A              & $64$          & $64$          & $64$      & $64$      & $64$      \\
Learning rate           & $0.1$              & $5 \times 10^{-4}$        & $5 \times 10^{-4}$        & $10^{-3}$     & $5 \times 10^{-4}$     & $5 \times 10^{-5}$     \\
Number of epochs        & $50$               & $50$          & $50$          & $50$      & $50$      & $50$      \\
Maximum number of gradient steps     & N/A              & $5000$        & $5000$        & $6000$    & $7500$    & $7500$    \\
$\beta_n$                    & $2.0$              & $2.0$        & $1.0$         & $1.0$     & $2.0$     & $2.0$ \\
\bottomrule
\end{tabular}
\end{threeparttable}\end{adjustbox}
\end{table}
\begin{table}[htb]
    \caption{Parameters of iCEM optimizer for experiments in~\Cref{sec:experiments}.}
    \label{tab:icem_params}
\begin{tabular}{lllll}
\midrule
\multicolumn{1}{c}{\textbf{Hyperparameters}} & \multicolumn{1}{c}{\textbf{Pendulum-v1 - GP}} & \multicolumn{1}{c}{\textbf{Pendulum-v1}} & \multicolumn{1}{c}{\textbf{Swimmer}} & \multicolumn{1}{c}{\textbf{Cheetah}} \\
\midrule
 Number of samples $P$             & 500                                  & 500                             & 250                         & 200                         \\
         Horizon $H$                             & 20                                   & 20                              & 15                          & 10                          \\
Size of elite-set $K$                & 50                                   & 50                              & 25                          & 20                          \\
Colored-noise exponent $\beta$            & 0.25                                 & 0.25                            & 0.25                         & 0.25                        \\
Number of particles                 & 10                                   & 10                              & 10                          & 10                          \\
\textit{CEM-iterations}                   & 10                                   & 10                              & 5                           & 5                           \\

Fraction of elites reused $\xi$              & 0.3                                  & 0.3                             & 0.3                         & 0.3      \\                  
\bottomrule
\end{tabular}
\end{table}
\paragraph{Model-based SAC optimizer}
For the reacher and MountainCar environment we use scarce rewards (c.f.,~\Cref{tab:environment_rewards}), which require long-term planning. Therefore, a receding horizon MPC approach is less suitable for these tasks. Accordingly, we use a policy network which we train using SAC~\citep{sac}. Moreover, our model-based SAC uses transitions simulated through the learned model to train a policy. Accordingly, given a dataset of transitions from the true environment, we sample $P$ initial states from the dataset. For each of the states, we simulate a trajectory of $H$ steps using our learned model and the SAC policy. We collect the simulated transitions into a simulation data buffer $\setD_{\text{SIM}}$, which we then use to perform $G$ gradient steps as suggested by \cite{sac} to train the policy. The algorithm is summarized below, and we provide the SAC hyperparameters in~\Cref{tab:sac_params}.
\begin{algorithm}[H]
    \caption*{\textbf{Model-based SAC}}
    \begin{algorithmic}[]
        \STATE {\textbf{Init: }}{Stastistical Model $M = \left(\vmu_n, \vsigma_n, \beta_n(\delta)\right)$, Dataset of transitions $\setD_n$, initial policy $\pi_0$, Model Rollout steps $H$}
        \STATE{$\setD_{\text{SIM}} \leftarrow \emptyset$ \hspace{5cm} \ding{228} Initialize simulated transitions buffer}
        \FOR{Training steps $k=1, \ldots, K$}{
            \vspace{-0.5cm}
            \STATE {
            \begin{align*}
            &\vx_{0, 1:P} \sim \setD_n &&\text{\ding{228} Sample $P$ initial states from buffer}
             \end{align*}}}
               \FOR{Initial state $\vx_{0} \in \vx_{0, 1:P}$}{
             \STATE{
               \begin{align*}
                &\mathcal{D}_{\text{SIM}} \leftarrow \mathcal{D}_{\text{SIM}} \cup \textsc{ModelRollout}(\vpi_k, \vx_{0}, H, M) \quad &&\text{\ding{228} Collect simulated transitions}
            \end{align*}
            }}
              \ENDFOR
                  \STATE{Perform $G$ gradient updates on $\vpi_k$ as proposed in \cite{sac} using $\setD_{\text{SIM}}$.}
        \ENDFOR
    \end{algorithmic}
\end{algorithm}
\begin{table}[htb]
    \centering
    \caption{Parameters of model-based SAC optimizer for experiments in~\Cref{sec:experiments}.}
    \label{tab:sac_params}
\begin{tabular}{lll}
\midrule
\multicolumn{1}{c}{\textbf{Parameters}} & \multicolumn{1}{c}{\textbf{MountainCar}} & \multicolumn{1}{c}{\textbf{Reacher}} \\
\midrule
Discount factor                              & $0.99$                                     & $0.99$                                 \\
Learning rate actor                          & $5 \times 10^{-4}$                                     & $5 \times 10^{-4}$                               \\
Learning rate critic                         & $5 \times 10^{-4}$                                     & $5 \times 10^{-4}$                               \\
Learning rate entropy coefficient            & $5 \times 10^{-4}$                                     & $5 \times 10^{-4}$                                  \\
Actor architecture                           & $[64, 64]$                             & $[250, 250]$                       \\
Critic architecture                          & $[256, 256]$                           & $[250, 250]$                      \\
Batch size                                   & 32                                       & 64                                   \\
Gradient steps $G$                              & 350                                      & 500                                  \\
Simulation horizon   $H$                        & 4                                        & 10                                   \\
Initial state sample size   $P$                 & 500                                      & 2000                                 \\
Total SAC training steps      $K$               & 35                                       & 350         \\
\bottomrule
\end{tabular}
\end{table}

\subsection{\opacex in the High-dimensional Fetch Pick \& Place Environment} \label{subsec:fetch_pick}
\subsubsection{Environment and Model Details}
In our experiments, we use an extension of the Fetch Pick \& Place environment to multiple objects as proposed in \cite{li2020towards} and further modified in \cite{Sancaktaretal22} with the addition of a large table. This is a compositional object manipulation environment with an end-effector-controlled robot arm.
The robot actions $\vu \in \mathbb{R}^4$ correspond to the gripper movement in Cartesian coordinates and the gripper opening/closing. The robot state $\vx_{\text{robot}} \in \mathbb{R}^{10}$ contains positions and velocities of the end-effector as well as the gripper-state (open/close) and gripper-velocity.
Each object's state $\vx_{\text{obj}} \in \mathbb{R}^{12}$ is given by its position, orientation (in Euler angles), and linear and angular velocities.

We follow the free-play paradigm as used in CEE-US \citep{Sancaktaretal22}. At the beginning of free play, an ensemble of world models is randomly initialized with an empty replay buffer. At each iteration of free play, a certain number of rollouts (here: 20 rollouts with 100 timesteps each) are collected and then added to the replay buffer. Afterwards, the models in the ensemble are trained for a certain number of epochs on the collected data so far. Note that unlike in the original proposal by \citet{Sancaktaretal22}, we use Multilayer Perceptrons (MLP) as world models instead of Graph Neural Networks (GNN), for the sake of reducing run-time. This corresponds to the ablation \mbox{MLP\,+\,iCEM} presented in \cite{Sancaktaretal22}, which was shown to outperform all the baselines other than CEE-US with GNNs. As we are interested in exploring the difference in performance via injection of optimism into active exploration, we use the computationally less heavy MLPs in our work. This is reflected in the downstream task performance we report compared to the original CEE-US with GNNs. Details for the environment and models are summarized in \Cref{tab:environment_parameters}.

After the free-play phase, we use the trained models to solve downstream tasks zero-shot via model-based planning with iCEM. We test for the tasks pick \& place, throwing and flipping with 4 objects. The rewards used for these tasks are the same as presented in \cite{Sancaktaretal22}.

\begin{table*}[htb]
    \centering
    \caption{Environment and model settings used for the experiment results shown in \Cref{fig:construction:success_rates}.}
    \label{tab:environment_parameters}
    \begin{subtable}[t]{.45\textwidth}
        \centering
        \caption{Fetch Pick \& Place settings.}
        \begin{tabular}{@{}ll@{}}
        \toprule
        \textbf{Parameter} & \textbf{Value} \\
        \midrule
        Episode Length & $100$ \\
        Train Model Every & $20$ Ep. \\
        Action Dim. & $4$ \\
        Robot/Agent State Dim. & $10$ \\
        Object Dynamic State Dim. & $12$ \\
        Number of Objects & $4$\\
        \bottomrule
        \end{tabular}
    \end{subtable}
    \begin{subtable}[t]{.45\textwidth}
        \centering
        \caption{MLP model settings.}
        \begin{tabular}{@{}ll@{}}
        \toprule
        \textbf{Parameter} & \textbf{Value} \\
        \midrule
        Network Size & $3 \times 256$ \\
        Activation function & SiLU \\
        Ensemble Size & 5\\
        Optimizer & ADAM \\
        Batch Size & $256$ \\
        Epochs & $50$ \\
        Learning Rate & $0.0001$ \\
        Weight decay & $5 \cdot 10^{-5}$ \\
        Weight Initialization & Truncated Normal\\
        Normalize Input & Yes \\
        Normalize Output & Yes \\
        Predict Delta & Yes \\
        \bottomrule
        \end{tabular}
    \end{subtable}%
\end{table*}

\subsubsection{\opacex Heuristic Variant}
In the case of Fetch Pick \& Place with an high-dimensional observation space, we implement a heuristic of \opacex. Note that as Fetch Pick \& Place is a deterministic environment without noise, we only model epistemic uncertainty. 

\begin{algorithm}[H]
    \caption*{\textbf{\opacex (Heuristic Variant)}}
    \label{algorithm:heuristicOpAcEx}
    \begin{algorithmic}[]
        \STATE {\textbf{Input: }}{Ensemble $\{\vf_{i}\}_{k=1}^{K}$, $\epsilon \ll 1$}
        \FOR{$n \in 1,\ldots,N_{\text{max}}$}{
            \STATE{Solve optimal control problem till convergence for the system: $\vx_{t+1} = \vf_{j_t}(\vx_{t}, \vu_{t})$.}
            \STATE{
            \begin{align*}
            \vu^{\star}_{0:T-1}, j^{\star}_{0:T-1} = \argmax_{\vu_{0:T-1},j_{0:T-1}} \sum_{t=0}^{T-1} \sum_{i=1}^{d_{x}} \log\left( \epsilon^2 + \sigma_{n,i}^{2}(\vx_{t}, \vu_{t}) \right) &&\text{\ding{228} Estimate }\\
            &&\text{\hspace{-8em}} \sigma_{n,i}(\vx_{t}, \vu_{t}) \text{ with ensemble disagreement.}
            \end{align*}
            }
            \STATE{Rollout $\vu^{\star}_{0:T_1}$ on the system and collect data $\setD_{n}.$}
            \STATE{Update models $\{\vf_{i}\}_{k=1}^{K}.$}
        }
        \ENDFOR
    \end{algorithmic}
\end{algorithm}

\subsubsection{Controller Parameters}
The set of default hyperparameters used for the iCEM controller are presented in \Cref{tab:controller_settings}, as they are used during the intrinsic phase for \opacex, CEE-US, and other baselines.
The controller settings used for the extrinsic phase are given in \Cref{tab:extrinsic_controller}. For more information on the hyperparameters and the iCEM algorithm, we refer the reader to \cite{iCem}.

\begin{table*}[htb]
    \centering
    \caption{Base settings for iCEM as they are used in the intrinsic phase. Same settings are used for all methods.}
    \label{tab:controller_settings}
    \begin{subtable}[t]{.48\textwidth}
        \centering
        \begin{tabular}{@{}ll@{}}
        \toprule
        \textbf{Parameter} & \textbf{Value} \\
        \midrule
        Number of samples $P$ & $128$ \\
        Horizon $H$ & $30$ \\
        Size of elite-set $K$ & $10$ \\
        Colored-noise exponent $\beta$ & $3.5$ \\
        \textit{CEM-iterations} & $3$ \\
        Noise strength $\sigma_{\text{init}}$ & $0.5$ \\
        Momentum $\alpha$ & $0.1$ \\
        \texttt{use\_mean\_actions} & Yes \\
        \texttt{shift\_elites} & Yes \\
        \texttt{keep\_elites} & Yes \\
        Fraction of elites reused $\xi$ & $0.3$ \\
        Cost along trajectory & \texttt{sum} \\
        \bottomrule
        \end{tabular}
    \end{subtable}%
\end{table*}
\begin{table}[htb]
    \centering
  \caption{iCEM hyperparameters used for zero-shot generalization in the extrinsic phase. Any settings not specified here are the same as the general settings given in \Cref{tab:controller_settings}.}
    \label{tab:extrinsic_controller}
    \renewcommand{\arraystretch}{1.03}
    \resizebox{1.\linewidth}{!}{ 
    \begin{tabular}{@{}l|ccccc@{}}
    \toprule
    \textbf{Task}     & \multicolumn{5}{c}{\textbf{Controller Parameters}} \\
                      & Horizon & Colored-noise exponent  & \texttt{use\_mean\_actions} & Noise strength & Cost Along\\
                      & $h$ & $\beta$ &  & $\sigma_\text{init}$ & Trajectory \\
    \midrule
    Pick \& Place &    30  &  3.5  &  Yes  &  0.5 &    \texttt{best}     \\
    Throwing      &    35  &  2.0  &  Yes  &  0.5 &    \texttt{sum}     \\
    Flipping      &    30  &  3.5  &   No  &  0.5 &    \texttt{sum}     \\
  \bottomrule
    \end{tabular}
   }

\end{table}
\clearpage
\section{Study of exploration intrinsic rewards} 
\label{sec:study_intrinsic_reward}
\begin{figure}[h]
    \centering
\includegraphics[width=0.8\textwidth]{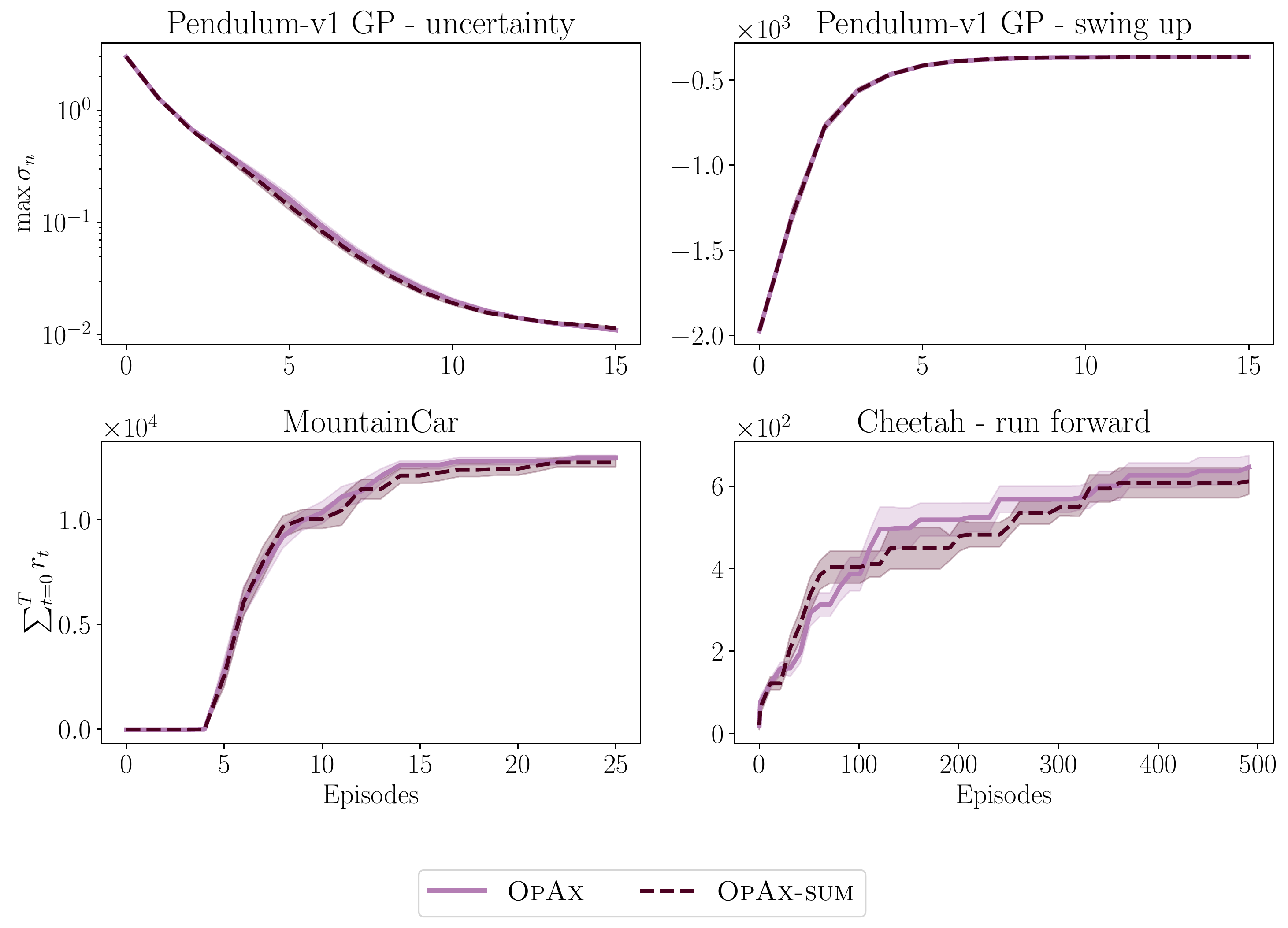}
    \caption{Comparison of \opacex between intrinsic reward proposed in \Cref{eq:exploration_op} and the sum of epistemic uncertainties proposed by~\citep{pathak2019self}, i.e., \opacex-\textsc{sum}. For both choices of intrinsic rewards, \opacex reduces the epistemic uncertainty and performs well on downstream tasks.}
    \label{fig:comparison_sum}
\end{figure}
The intrinsic reward suggested in \Cref{eq:exploration_op} takes the log of the model epistemic uncertainty. Another common choice for the intrinsic reward is the epistemic uncertainty or model disagreement without the log~\citep{pathak2019self, sekar2020planning}. In the following Lemma, we show that these rewards are closely related.
\begin{lemma}
    Let  $\sigma_{\max} = \sup_{\vz \in \setZ; i \geq 0; 1\leq j\leq d_x} \sigma_{i, j}(\vz)$ and $\sigma >  0$. Then for all $i\geq 0$ and $j \in \{1, \dots, d_x\}$
    \begin{equation*}
        \sigma^2_{i, j}(\vz) \leq \frac{\sigma_{\max}}{\log(1 + \sigma^{-2}\sigma_{\max})}\log(1 + \sigma^{-2}\sigma^2_{i, j}(\vz)) \leq \frac{\sigma^{-2}\sigma_{\max}}{\log(1 + \sigma^{-2}\sigma_{\max})} \sigma^2_{i, j}(\vz). 
    \end{equation*}
\end{lemma}
\begin{proof}
\cite{curi2020efficient}  derive the first inequality on the left. The second inequality follows since $\log(1 + x) \leq x$ for all $x \geq 0$.
\end{proof}
\looseness=-1
Due to this close relation between the two objectives, our theoretical findings also apply to the intrinsic reward proposed by~\citep{pathak2019self}. Moreover, empirically we notice in~\Cref{fig:comparison_sum} that \opacex performs similarly when the sum of epistemic uncertainties is used instead of the objective in \Cref{eq:exploration_op}. However, our objective can naturally be extended to the case of heteroscedastic aleatoric noise, since it trades off the ratio of epistemic and aleatoric uncertainty.

\end{document}